\documentclass[letterpaper]{article} 
\usepackage{aaai24}  
\usepackage{times}  
\usepackage{helvet}  
\usepackage{courier}  
\usepackage[hyphens]{url}  
\usepackage{graphicx} 
\usepackage{natbib}  
\usepackage{caption} 
\usepackage{algorithm}
\usepackage{algorithmic}

\newif\ifneedappendix
\needappendixfalse

\needappendixtrue 

\let\proof\relax

\usepackage{xspace}

\usepackage{amsthm}
\theoremstyle{definition}

\usepackage{epsfig,amsmath,amsfonts,epsfig,multirow,makecell,caption,soul,csquotes,subcaption,mathtools,bm,spverbatim,booktabs,tcolorbox,diagbox,todonotes}
\usepackage[e]{esvect}

\usepackage{latexsym}
\usepackage{amsmath}
\usepackage{amssymb}
\usepackage{bm}
\usepackage{mathrsfs}
\usepackage{url}

\newcommand*{\argmax}{\mathop{\mathrm{argmax}}}

\ifx\BlackBox\undefined
\newcommand{\BlackBox}{\rule{1.5ex}{1.5ex}}  
\fi

\ifx\QED\undefined
\def\QED{~\rule[-1pt]{5pt}{5pt}\par\medskip}
\fi

\ifx\proof\undefined
\newenvironment{proof}{\par\noindent{\bf Proof\ }}{\hfill\BlackBox\\[2mm]}
\fi

\ifx\theorem\undefined
\newtheorem{theorem}{Theorem}

\newtheorem{definition}[theorem]{Definition}

\fi

\ifx\axiom\undefined
\newtheorem{axiom}[theorem]{Axiom}
\fi

\newcommand{\ie}{\emph{i.e.}}

\renewcommand{\url}[1]{{\sffamily #1}}

\newcommand{\norm}[1]{\left\lVert #1 \right\rVert}

\newcommand{\header}[1]{\noindent\textbf{#1}}

\newcommand{\addfigs}[5]{
\begin{subfigure}{#5\linewidth}
    \def\vardataset{#1}
    \def\varradius{#2}
    \def\varindex{#3}
    \def\varcap{#4}
    \centering
    \begin{subfigure}{0.15\linewidth}
        \includegraphics[width=\linewidth]{./fig/\vardataset/\varindex/clean.png}
        \subcaption*{Clean}
    \end{subfigure}
    \hspace{0.001\linewidth}
    \begin{subfigure}{0.15\linewidth}
        \includegraphics[width=\linewidth]{./fig/\vardataset/\varindex/\varradius/noisy_em.png}
        \includegraphics[width=\linewidth]{./fig/\vardataset/\varindex/\varradius/em.png}
        \subcaption*{EM}
    \end{subfigure}
    \begin{subfigure}{0.15\linewidth}
        \includegraphics[width=\linewidth]{./fig/\vardataset/\varindex/\varradius/noisy_tap.png}
        \includegraphics[width=\linewidth]{./fig/\vardataset/\varindex/\varradius/tap.png}
        \subcaption*{TAP}
    \end{subfigure}
    \begin{subfigure}{0.15\linewidth}
        \includegraphics[width=\linewidth]{./fig/\vardataset/\varindex/\varradius/noisy_shortcut.png}
        \includegraphics[width=\linewidth]{./fig/\vardataset/\varindex/\varradius/shortcut.png}
        \subcaption*{SC}
    \end{subfigure}
    \begin{subfigure}{0.15\linewidth}
        \includegraphics[width=\linewidth]{./fig/\vardataset/\varindex/\varradius/noisy_rem.png}
        \includegraphics[width=\linewidth]{./fig/\vardataset/\varindex/\varradius/rem.png}
        \subcaption*{REM}
    \end{subfigure}
    \begin{subfigure}{0.15\linewidth}
        \includegraphics[width=\linewidth]{./fig/\vardataset/\varindex/\varradius/noisy_ours.png}
        \includegraphics[width=\linewidth]{./fig/\vardataset/\varindex/\varradius/ours.png}
        \subcaption*{SEM}
    \end{subfigure}
    \ifthenelse{\equal{\varcap}{}}{}{\subcaption{\varcap}}
\end{subfigure}
}

\usepackage{newfloat}
\usepackage{listings}
\DeclareCaptionStyle{ruled}{labelfont=normalfont,labelsep=colon,strut=off} 
\floatstyle{ruled}
\newfloat{listing}{tb}{lst}{}
\floatname{listing}{Listing}
\usepackage{amsfonts,bm}
\usepackage{times}
\usepackage{latexsym}
\usepackage{url}
\usepackage{amsmath}
\usepackage{multirow}
\usepackage{colortbl}
\usepackage{tabularx,paralist}
\usepackage{tcolorbox}
\usepackage{adjustbox}
\usepackage{transparent}
\usepackage{booktabs}
\usepackage{tikz-dependency}
\usepackage{verbatim}
\usepackage{amssymb}
\usepackage{caption, subcaption}
\usepackage[T1]{fontenc}
\usepackage{enumitem}
\usepackage[utf8]{inputenc}
\usepackage{microtype}
\definecolor{grey}{rgb}{0.1,0.1,0.1}

\title{Stable Unlearnable Example: Enhancing the Robustness of Unlearnable Examples via Stable Error-Minimizing Noise
}
\author {
    Yixin Liu \textsuperscript{\rm 1},
	Kaidi Xu  \textsuperscript{\rm 2}, 
    Xun Chen \textsuperscript{\rm 3},
    Lichao Sun\textsuperscript{\rm 1},
}
\affiliations {
    \textsuperscript{\rm 1} Lehigh University, Bethlehem, Pennsylvania, USA\\
\textsuperscript{\rm 2} Drexel University,  Philadelphia, Pennsylvania, USA\\
\textsuperscript{\rm 2} Samsung Research America, Mountain View, California, USA\\
     \{yila22, lis221\}@lehigh.edu, kx46@drexel.edu, Xun.chen@samsung.com
}

\begin{document}

\maketitle

\begin{abstract}
The open sourcing of large amounts of image data promotes the development of deep learning techniques. Along with this comes the privacy risk of these image datasets being exploited by unauthorized third parties to train deep learning models for commercial or illegal purposes. To avoid the abuse of data, a poisoning-based technique, ``unlearnable example'', has been proposed to significantly degrade the generalization performance of models by adding imperceptible noise to the data. To further enhance its robustness against adversarial training, existing works leverage iterative adversarial training on both the defensive noise and the surrogate model. However, it still remains unknown whether the robustness of unlearnable examples primarily comes from the effect of enhancement in the surrogate model or the defensive noise. Observing that simply removing the adversarial noise on the training process of the defensive noise can improve the performance of robust unlearnable examples, we identify that solely the surrogate model's robustness contributes to the performance. Furthermore, we found a negative correlation exists between the robustness of defensive noise and the protection performance, indicating defensive noise's instability issue. Motivated by this, to further boost the robust unlearnable example, we introduce Stable Error-Minimizing noise (SEM), which trains the defensive noise against random perturbation instead of the time-consuming adversarial perturbation to improve the stability of defensive noise. Through comprehensive experiments, we demonstrate that SEM achieves a new state-of-the-art performance on CIFAR-10, CIFAR-100, and ImageNet Subset regarding both effectiveness and efficiency. The code is available at \underline{https://github.com/liuyixin-louis/Stable-Unlearnable-Example}.
\end{abstract}

\section{Introduction}
The proliferation of open-source and large-scale datasets on the Internet has significantly advanced deep learning techniques across various fields, including computer vision \cite{he2016deep, dosovitskiy2020image, cao2023comprehensive}, natural language processing \cite{devlin2018bert, vaswani2017attention, zhou2023comprehensive}, and graph data analysis \cite{wang2018non, kipf2016semi, sun2022adversarial}. However, this emergence also poses significant privacy threats, as these datasets can be exploited by unauthorized third parties to train deep neural networks (DNNs) for commercial or even illegal purposes. For instance, \citet{hill2019photos} reported that a tech company amassed extensive facial data without consent to develop commercial face recognition models.
\begin{figure}[t]
    \centering
    \includegraphics[width=0.9\linewidth]{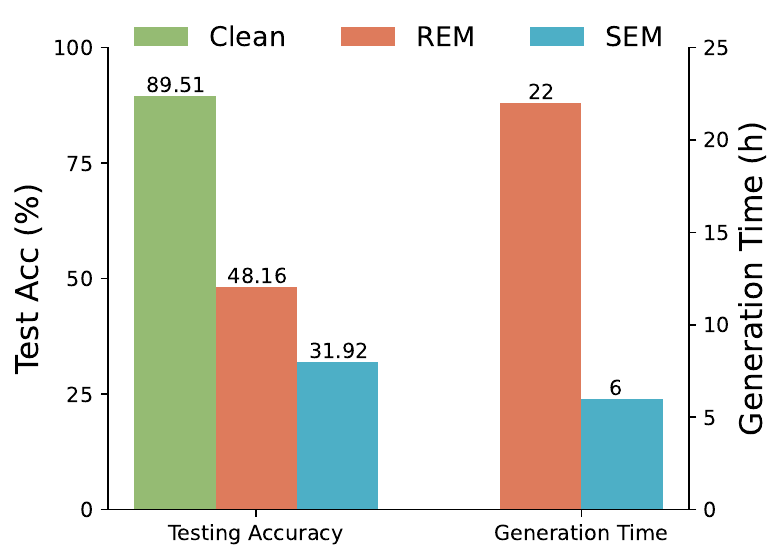}
    \caption{
        The performance comparison on CIFAR-10 between the current SoTA method, the robust error-minimizing noise (REM) \cite{fuRobustUnlearnableExamples2022}, and our proposed stable error-minimizing noise (SEM). Our SEM outperforms the REM in terms of both effectiveness and generation efficiency.
        }
    \label{fig: performance}
\end{figure}

To tackle this problem, recent works have introduced \textit{Unlearnable Examples} \cite{fowl2021adversarial, huang2020unlearnable, fowl2021preventing}, which aims to make the training data ``unlearnable'' for deep learning models by adding a type of invisible noise. This added noise misleads the learning model into adopting meaningless shortcut patterns from the data rather than extracting informative knowledge \cite{yuAvailabilityAttacksCreate2022}. However, such conferred unlearnability is vulnerable to adversarial training~\cite{huang2020unlearnable}.
In response, the concept of robust error-minimizing noise (REM) was proposed in \cite{fuRobustUnlearnableExamples2022} to enhance the efficacy of error-minimizing noise under adversarial training, thereby shielding data from adversarial learners by minimizing adversarial training loss. Specifically, a min-min-max optimization strategy is employed to train the robust error-minimizing noise generator, which in turn produces robust error-minimizing noise. To enhance stability against data transformation, REM leverages the Expectation of Transformation (EOT) \cite{athalye2018synthesizing} during the generator training. The resulting noise demonstrates superior performance in advanced training that involves both adversarial training and data transformation.

However, a primary drawback of REM, is its high computational cost. Specifically, on the CIFAR-10 dataset, it would take nearly a full day to generate the unlearnable examples, which is very inefficient. Consequently, enhancing the efficiency of REM is vital, especially when scaling up to larger real-world datasets like ImageNet \cite{russakovsky2015imagenet}. To improve the efficiency of robust unlearnable example generation algorithms, in this paper, we take a closer look at the time-consuming adversarial training process in both the surrogate model and the defensive noise. As empirically demonstrated in Tab. \ref{tab: comparison}, we can see that the performance of the robust unlearnable example mainly comes from the effect of adversarial training on the surrogate model rather than the defensive noise part. Surprisingly, the presence of adversarial perturbation in the defensive noise crafting will even lead to performance degradation, indicating that we need a better optimization method in this part. 

To elucidate this intriguing phenomenon and enhance the robustness of unlearnable examples, we begin by defining the robustness of both the surrogate model and defensive noise. Subsequently, our correlation analysis reveals that the robustness of the surrogate model is the primary contributing factor. Conversely, we observe a negative correlation between the robustness of defensive noise and data protection performance. We hypothesize that the defensive noise overfits to monotonous adversarial perturbations, leading to its instability. To address this issue, we introduce a novel noise type, \textit{stable error-minimizing noise} (SEM). Our SEM is trained against random perturbations, rather than the more time-consuming adversarial perturbations, to enhance stability. We summarize our contributions as follows:
\begin{itemize}
    \item We establish that the robustness of unlearnable examples is largely attributable to the surrogate model's robustness, rather than that of the defensive noise. Furthermore, we find that adversarially enhancing defensive noise can actually degrade its protective performance.
    \item To mitigate such an instability issue, we introduce stable error-minimizing noise (SEM), which trains the defensive noise against random perturbations instead of the more time-consuming adversarial ones, to improve the stability of the defensive noise.
    \item Extensive experiments empirically demonstrate that SEM achieves a SoTA performance on CIFAR-10, CIFAR-100, and ImageNet Subset regarding both effectiveness and efficiency. Notably, SEM achieves a $3.91\times$ speedup and an approximately $17\%$ increase in testing accuracy for protection performance on CIFAR-10 under adversarial training with $\epsilon = 4/255$.
\end{itemize}

\begin{figure*}[!thbp]
    \centering
    \includegraphics[width=0.9\linewidth]{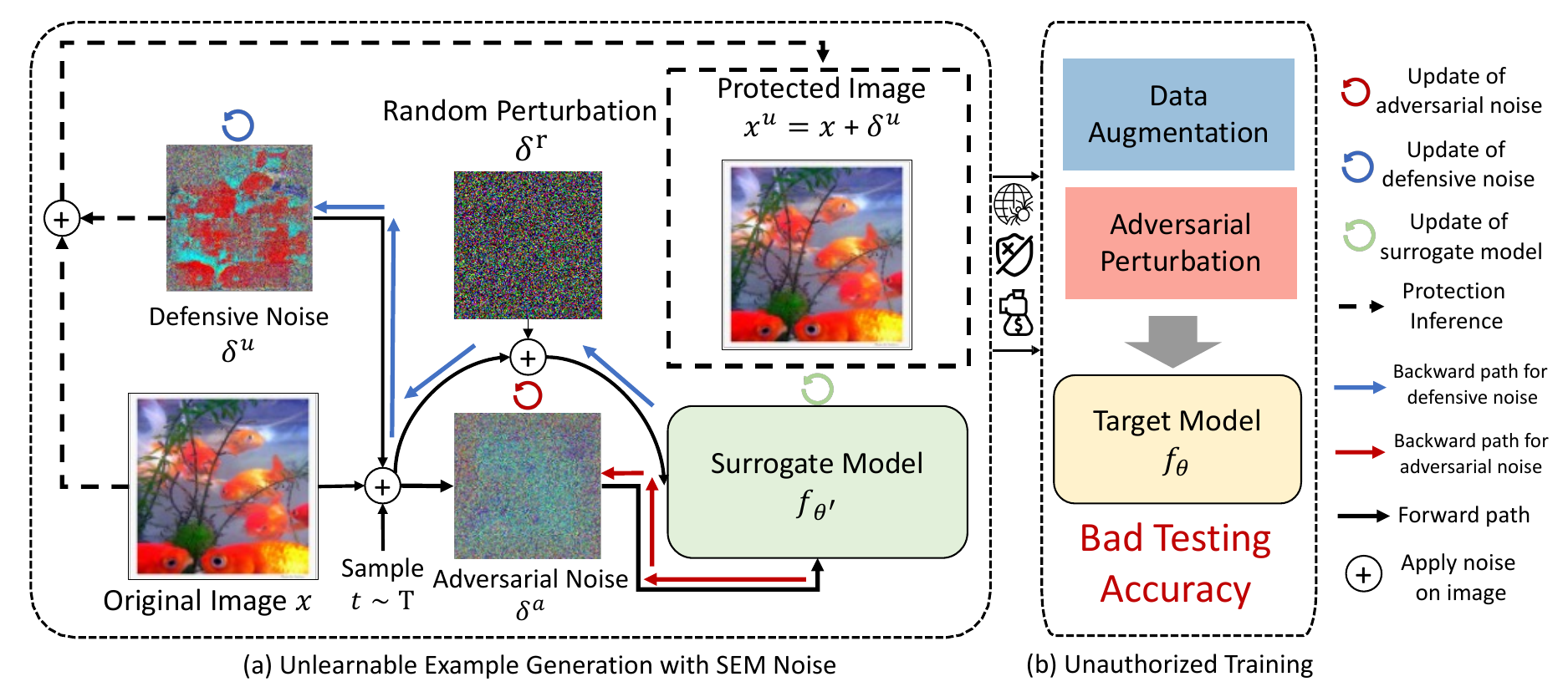}
    \caption{The overall framework of our approach. Our approach consists of two phases: noise training and generator training. During the noise training phase, we train the defensive noise, denoted as $\delta^u$, to counter random perturbations. In the subsequent generator training phase, the original images, represented as $x$, are transformed to $x_{\text{input}} = t(x + \delta^u) + \delta^a$ before being input into the network. Here, $t$ represents a transformation derived from distribution $T$, and $\delta^a$ represents the adversarial perturbation produced using PGD. The noise generator, $f^{\prime}_{\theta}$, updates the network parameters, $\theta$, by minimizing adversarial loss. By applying our defensive noise, models trained on the protected data learn minimal information and exhibit poor performance on clean data.
    }
    \label{fig:framework}
\end{figure*}

\section{Preliminaries}
\label{pre}
\header{Setup.} 
We consider a classification task with input-label pairs from a $K$-class dataset $\mathcal{T}={(x_i,y_i)}_{i=1}^{n}\subseteq \mathcal{X} \times \mathcal{Y}$, where $\mathcal{T}$ is constructed by drawing from an underlying distribution $\mathcal{D}(x,y)$ in an \textit{i.i.d.} manner, and $x\in\mathbb{R}^{d}$ represents the features of a sample. 
Let $f_{\theta}:\mathbb{R}^{d}\to\Delta$ be a neural network model that outputs a probability simplex, e.g., via a softmax layer.
In most learning algorithms, we employ empirical risk minimization (ERM), which aims to train the model $f_{\theta}$ by minimizing a loss function $\mathcal{L}(f(x), y)$ on the clean training set. This is achieved by solving the following optimization:
\begin{equation}
	\min_{\theta} \frac{1}{n} \sum_{i=1}^{n} \mathcal{L}\left(f_{\theta}\left(x_{i}\right), y_{i}\right).
\end{equation}

\begin{table}
    \centering
    \caption{Comparison between variants of REM and SEM on CIFAR-10 dataset. The $\checkmark$ indicates the corresponding method is used in the training process. $\delta^r$ indicates the defensive noise that is trained against random perturbation.}
    \label{tab: comparison}
    \begin{tabular}{cccccc} 
        \toprule
        \multirow{2}{*}{Method} & \multicolumn{2}{c}{Adv. Train} & \multirow{2}{*}{$\delta^r$} & \multirow{2}{*}{Time} & \multirow{2}{*}{Test Acc. (\%) $\downarrow$}  \\ 
        \cmidrule(l){2-3}
                                & $\theta$     & $\delta^u$      &                             &                       &                                           \\ 
        \midrule
        REM                     & $\checkmark$ & $\checkmark$    &                             & 22h                   & $46.72$                                   \\
        REM-$\delta^u$          &              & $\checkmark$    &                             & 15h                   & $88.28$                                   \\
        REM-$\theta$            & $\checkmark$ &                 &                             & 6h                    & $37.90$                                   \\ 
        \midrule
        SEM                     & $\checkmark$ &                 & $\checkmark$                & \textbf{6h}           & \textbf{30.26}                                   \\
        \bottomrule
        \end{tabular}
\end{table}

\header{Adversarial Training. } Adversarial training aims to enhance the robustness of models against adversarially perturbed examples~\cite{madry2018towards}. Specifically, adversarial robustness necessitates that $f_\theta$ performs well not only on $\mathcal{D}$ but also on the worst-case perturbed distribution close to $\mathcal{D}$, within a given adversarial perturbation budget. In this paper, we focus on the adversarial robustness of \emph{$\ell_\infty$-norm}: \ie, for a small $\epsilon > 0$, we aim to train a classifier $f_{\theta}$ that correctly classifies $(x+\delta, y)$ for any $\lVert\delta\rVert_{\infty}\leq\rho_a$\footnote{$\lVert\cdot\rVert$ in subsequent sections omits the subscript "$\infty$" for brevity.
}, where $(x, y)\sim \mathcal{D}$. Typically, adversarial training methods formalize the training of $f_\theta$ as a min-max optimization with respect to $\theta$ and $\delta$, \ie, 
\begin{equation}
	\min_{\theta} \frac{1}{n}\sum_{i=1}^n{\max_{\left\| \delta _i \right\| \le \rho _a}}\mathcal{L} \left( f_{\theta}\left( x_i+\delta _i \right) ,y_i \right) .
\end{equation}

\header{Unlearnable Example. } Unlearnable examples \cite{huang2020unlearnable} leverage clean-label data poisoning techniques to trick deep learning models into learning minimal useful knowledge from the data, thus achieving the data protection goal. By adding imperceptible defensive noise to the data, this technique introduces misleading shortcut patterns to the training process, thereby preventing the models from acquiring any informative knowledge. Models trained on these perturbed datasets exhibit poor generalization ability. Formally, this task can be formalized into the following bi-level optimization:
\begin{equation}
	\begin{gathered}
\max _{\norm{\delta^{u}_i} \le \rho_a} \underset{(x, y) \sim \mathcal{D}}{\mathbb{E}}\left[\mathcal{L} \left( f_{\theta ^*}(x),y \right) \right],\\
\,\,\mathrm{s}.\mathrm{t}.~\theta ^*=\underset{\theta}{\mathrm{arg}\min}\sum_{\left( x_i,y_i \right) \in \mathcal{T}}{\left[ \mathcal{L} \left( f_{\theta}\left( x_i+\delta _{i}^{u} \right) ,y_i \right) \right]}.
\end{gathered}
	\label{eq:unl}
\end{equation}

Here, we aim to maximize the empirical risk of trained models by applying the generated \textit{defensive perturbation} $\mathcal{P}^u=\left\{{\delta}_{i}^{u}\right\}_{i=1}^{n}$ into the original training set $\mathcal{T}$. Owing to the complexity of directly solving the bi-level optimization problem outlined in Eq. \ref{eq:unl}, several approximate methods have been proposed. These approaches include rule-based \cite{yuAvailabilityAttacksCreate2022}, heuristic-based \cite{huang2020unlearnable}, and optimization-based methods \cite{feng2019learning}, all of which achieve satisfactory performance in solving Eq. \ref{eq:unl}.

\header{Robust Unlearnable Example. } 
However, recent studies \cite{fuRobustUnlearnableExamples2022,huang2020unlearnable, tao2021better} have demonstrated that the effectiveness of unlearnable examples can be compromised by employing adversarial training. To further address this issue, the following robust unlearnable example generation problem is proposed, which can be illustrated as a two-player game consisting of a defender ${U}$ and an unauthorized user $\mathcal{A}$. The defender ${U}$ aims to protect data privacy by adding perturbation $\mathcal{P}^{u}$ to the data, thereby decreasing the test accuracy of the trained model, while the unauthorized user $\mathcal{A}$ attempts to use \textit{adversarial training} and \textit{data transformation} to negate the added perturbation and ``recover'' the original test accuracy. Based on \cite{fuRobustUnlearnableExamples2022}, we assume that the defender $U$ has complete access to the data they intend to protect and cannot interfere with the user’s training process after the protected images are released. Additionally, we assume, as per \cite{fuRobustUnlearnableExamples2022}, that the radius of defensive noise $\rho_u$ exceeds that of the adversarial training radius $\rho_a$, ensuring the problem's feasibility. Given a distribution $T$ over transformations ${t: \mathcal{X} \rightarrow \mathcal{X}}$, we have
\begin{equation}
\footnotesize{
\begin{array}{c}
	\mathcal{P} ^u=\mathop {\mathrm{arg}\max} \limits_{\forall i,|\delta _{i}^{u}|| \le \rho _u}\underset{(x,y)\sim \mathcal{D}}{\mathbb{E}}\left[ \mathcal{L} \left( f_{\theta ^*}(x),y \right) \right] ,\\
	\mathrm{s}.\mathrm{t}. \theta ^*=\underset{\theta}{\mathrm{arg}\min}\sum_{\left( x_i,y_i \right) \in \mathcal{T}}{\underset{t\sim T}{\mathbb{E}} \underset{||\delta _{i}^{a}|| \le \rho _a}{\max} \mathcal{L} \left( f_{\theta}\left( x_{i}^{\prime} \right) ,y_i \right)},\\
\end{array}
}
\label{eq:rob unl}
\end{equation}
where $x_i^\prime = t(x_i + \delta _{i}^{u}) + \delta _{i}^{a}$ represents the transformed data, with $\delta _{i}^{a}$ being the adversarial perturbation crafted using Projected Gradient Descent (PGD) and $\delta _{i}^{u}$ denoting the defensive noise. After applying $\mathcal{P}^u$, the protected dataset $\mathcal{T}^u=(x_i + \delta {i}^{u}, y_i)_{i=1}^{n}$ is obtained.

\section{Methodology}
\subsection{Iterative Training of Generator and Defensive Noise}
To address the problem presented in Eq. \ref{eq:rob unl}, REM introduces a robust noise-surrogate iterative optimization method, where a surrogate noise generator model, denoted as $\theta$, and the defensive noise, $\delta ^u$, are optimized alternately. 
From the model's perspective, the surrogate model $\theta$ is trained on iteratively \textit{perturbed poisoned data}, created by adding defensive and adversarial noise to the original training data, namely, $x^u_{\text{perturb}} = t(x + \delta^u) + \delta^a$. The surrogate model is trained to minimize the loss below to enhance its adversarial robustness:
\begin{equation}
    \theta ^{\prime}\gets \theta -\eta^\theta \nabla _{\theta}\mathcal{L} \left( f_{\theta}\left( t\left( x+\delta ^u \right) +\delta ^a \right) ,y \right).
    \label{eq: surrogate model}
\end{equation}
where $\eta^{(\cdot)}$ represents the learning rate, $t$ is a transformation sampled from $T$, and $\delta_a$ is the adversarial perturbation crafted using PGD, designed to maximize the loss. Conversely, the defensive noise $\delta ^u$ is trained to counteract the worst-case adversarial perturbation using PGD based on $\theta$. The main idea is that robust defensive noise should maintain effectiveness even under adversarial perturbations. Specifically, the defensive noise is updated by minimizing the loss,
\begin{equation}
    \delta ^{u}\gets \delta ^{u}-\eta ^u\nabla _{\delta ^{u}}\mathcal{L} \left( f_{\theta}\left( t\left( x+\delta ^u \right) +\delta ^a \right) ,y \right).
    \label{eq: def noise}
\end{equation}

During the optimization of the surrogate model $\theta$, the defensive noise $\delta^u$ is fixed. Conversely, when optimizing the defensive noise $\delta^u$, the surrogate model $\theta$ remains fixed. We repeat these steps iteratively until the maximum training step is reached. The final defensive noise is generated by, 
\begin{equation}
    \delta ^{u}=\underset{\left\| \delta ^{u} \right\| \le \rho _u}{\arg \min } \mathbb{E} _{t\sim T}\max _{\left\| \delta ^{a} \right\| \le \rho _a} \mathcal{L} \left( f_{\theta}\left( t\left( x+\delta ^u \right) +\delta ^a \right) ,y \right).
    \label{eq: final def noise}
\end{equation}

\begin{algorithm}[htbp]
\caption{Noise Generator Training for SEM approach}
\label{alg: sem-train}
\begin{algorithmic}[1]
\REQUIRE
Training data set $\mathcal{T}$, Training steps $M$, defense PGD parameters $\rho_u$, $\alpha_u$ and $K_u$, attack PGD parameters $\rho_a$, $\alpha_a$ and $K_a$, transformation distribution $T$, the sampling number $J$ for gradient approximation
\ENSURE Robust noise generator $f'_\theta$.
    \FOR{$i$ \textbf{in} $1, \cdots, M$}
        \STATE Sample minibatch $(x, y) \sim \mathcal{T}$, randomly initialize $\delta^u$.
        \FOR{$k$ \textbf{in} $1,\cdots, K_u$}
            \FOR{$j$ \textbf{in} $1,\cdots, J$}
                \STATE Sample noise and transformation $\delta^r_j \sim \mathcal{P}$, $t_j \sim T$.
            \ENDFOR
            \STATE $g_k \leftarrow \frac{1}{J} \sum_{j=1}^J \frac{\partial}{\partial \delta^u} \ell(f'_\theta(t_j(x+\delta^u) + \delta^r_j), y)$
            \STATE $\delta^u \leftarrow \prod_{\|\delta\|\leq\rho_u} \left( \delta^u - \alpha_u \cdot \mathrm{sign}(g_k) \right)$
        \ENDFOR
        \STATE Sample a transformation function $t \sim T$.
        \STATE $\delta^a \leftarrow \mathrm{PGD}(t(x+\delta^u),y,f'_\theta,\rho_a,\alpha_a,K_a)$ 
		\STATE Update source model $f'_\theta$ based on minibatch $(t(x+\delta^u)+\delta^a,y)$
    \ENDFOR
\end{algorithmic}
\label{alg:train}
\end{algorithm}

\subsection{Improving the Stability of Defensive Noise}
\label{sec: stability}
The adversarial perturbation, denoted as $\delta ^a $, is incorporated in the optimization processes of \textit{both} the surrogate model $\theta$ and the defensive noise $\delta ^u$ (refer to Eq. \ref{eq: surrogate model} and Eq. \ref{eq: def noise}). However, as indicated by the results in Tab. \ref{tab: comparison}, it appears that solely the adversarial perturbation $\delta ^a $ in the surrogate model $\theta$'s optimization contributes to the performance enhancement. Surprisingly, the adversarial perturbation in the optimization of defensive noise $\delta ^u$ (see Eq. \ref{eq: def noise}) results in \textit{performance degradation}. To elucidate these intriguing findings, we proceed with a correlation analysis as described below. We initially define the protection performance by $F=1-{\tt Acc}$, where ${\tt Acc}$ represents the testing accuracy of the trained model. Then, we propose the following definition to quantify the robustness of both the surrogate model and the defensive noise. 

\begin{definition}[Robustness of surrogate model]
\label{def: robustness of surrogate model}
Given a fixed surrogate model, denoted as $\theta$, we define its robustness as the model's resistance to adversarial perturbations, where the perturbation $\delta^u$ is updateable,
\begin{equation}
    \mathcal{R} _{\theta}= - \max _{\left\| \delta ^{a} \right\| \le \rho _a} \min _{\left\| \delta ^{u} \right\| \le \rho _u} \mathcal{L} \left( f_{\theta}\left( t\left( x+\delta ^u \right) +\delta ^a \right) ,y \right).
    \label{eq: surrogate robustness}
\end{equation}
\end{definition}
\begin{definition}[Robustness of defensive noise]
\label{def: robustness of defensive noise}
For a given fixed defensive noise, $\delta^u$, we define its robustness as the noise's resistance to adversarial perturbations, where the model parameter $\theta$ is updateable,
\begin{equation}
      \mathcal{R} _{\delta ^u}=-\max _{\left\| \delta ^{a} \right\| \le \rho _a} \min _\theta \mathcal{L} \left( f_{\theta}\left( t\left( x+\delta ^u \right) +\delta ^a \right) ,y \right).
    \label{eq: noise robustness}
\end{equation}
\end{definition}

To explore the correlation between the two formulated robustness terms, $\mathcal{R}$, and the protection performance, $F$, we followed the standard defensive noise training procedure as outlined in \cite{fuRobustUnlearnableExamples2022} and stored the surrogate model, $\theta_t$, and defensive noise, $\delta^u_t$, at various training steps denoted by $t$. Subsequently, for the obtained model or noise, we fixed one while randomly initializing the other, then solved Eq. \ref{eq: surrogate robustness} and Eq. \ref{eq: noise robustness} through an iterative training process.
\begin{figure}[thbp]
    \centering 
    \includegraphics[width=\linewidth]{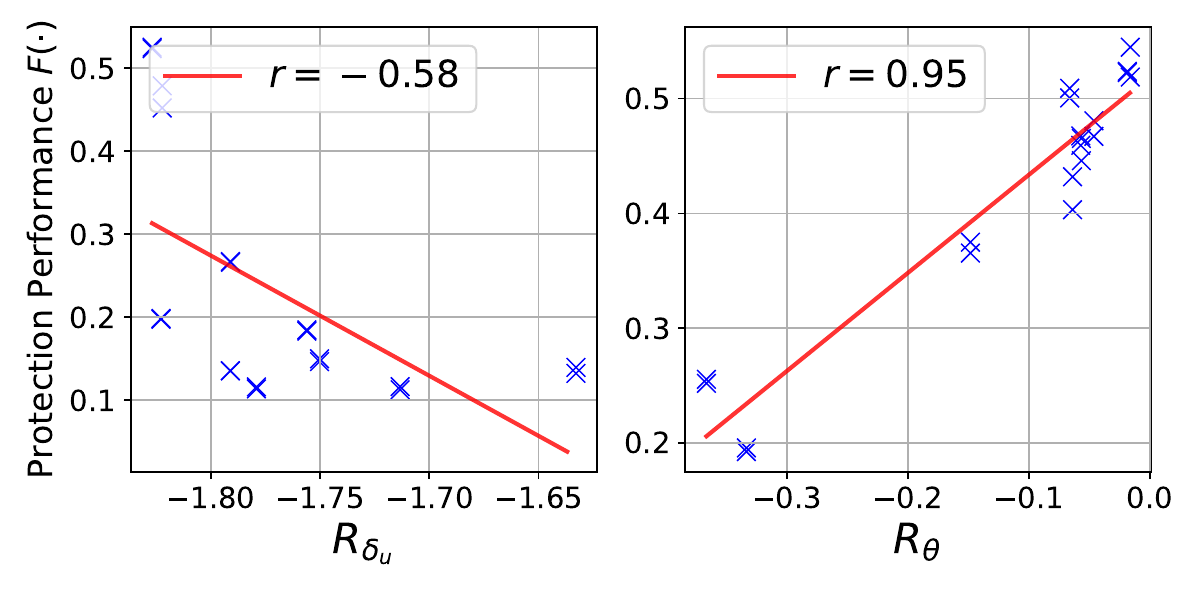}
    \caption{Exploration of the contribution of the robustness of defensive noise, denoted as $\mathcal{R}_{\theta}$, and the surrogate model, represented by $\mathcal{R}_{\delta^u}$, to the protection performance $F$. The Pearson correlation coefficients ($r$) quantify the strength of these relationships. Tests were conducted on the CIFAR-10 dataset with settings $\rho_a = 4/255$ and $\rho_u = 8/255$.
    }
    \label{fig: corr}
\end{figure}
From Fig. \ref{fig: corr}, we found that the \( \mathcal{R}_{\theta} \) demonstrates a strong positive correlation with the protection performance while \( \mathcal{R}_{\delta_u} \) displays a moderate negative correlation with the protection performance. 

This suggests that the protection performance of defensive noise is primarily and positively influenced by the robustness of the surrogate model, $\mathcal{R}_\theta$. Conversely, enhancing the robustness of defensive noise may paradoxically impair its protection performance. We hypothesize that the instability of defensive noise $\delta ^u$, trained following Eq. \ref{eq: def noise}, stems from the monotonic nature of the worst-case perturbation $\delta ^a$. To enhance its stability, we propose an alternative training objective for defensive noise,
\begin{equation}
    \delta ^{u}\gets \delta ^{u}-\eta ^u\nabla _{\delta ^{u}}\mathcal{L} \left( f_{\theta}\left( t\left( x+\delta ^u \right) +\delta^r \right) ,y \right),
    \label{eq: sem}
\end{equation}
where $\delta^r$ represents a random perturbation sampled from the uniform distribution $\mathcal{U}(-\rho_r, \rho_r)$. The radius of the random perturbation, $\rho_r$, is set to match that of the adversarial perturbation, $\rho_a$. Substituting $\delta^a$ in Eq. \ref{eq: def noise} with $\delta^r$ enables the crafted defensive noise to experience more diverse perturbations during training. We term the obtained defensive noise as stable error-minimizing noise (SEM). The overall framework and procedure are depicted in Fig. \ref{fig:framework} and Alg. \ref{alg: sem-train}.
\begin{table*}[t]
    \centering
    \caption{
    Test accuracy (\%) of models trained on data protected by different defensive noises under adversarial training with various perturbation radii. The defensive perturbation radius $\rho_u$ is globally set at $8/255$, while the adversarial perturbation radius $\rho_a$ of REM varies. The lower the test accuracy, the better the effectiveness of the protection.
    }
    \resizebox{1.0\textwidth}{!}{
    \begin{tabular}{c|ccccc|ccccc|ccccc} 
        \toprule
        Datasets$\rightarrow$ & \multicolumn{5}{c}{CIFAR-10}                                                       & \multicolumn{5}{c}{CIFAR-100}                                                   & \multicolumn{5}{c}{ImageNet Subset}                                                                  \\ 
        \cmidrule(lr){1-1}\cmidrule(l){2-16}
        Methods$\downarrow$   & $\rho_a$=0     & 1/255          & 2/255          & 3/255          & 4/255          & 0              & 1/255         & 2/255         & 3/255         & 4/255          & 0             & 1/255          & 2/255                & 3/255                & 4/255                 \\ 
        \midrule
        Clean Data            & 94.66          & 93.74          & 92.37          & 90.90          & 89.51          & 76.27          & 71.90         & 68.91         & 66.45         & 64.50          & 80.66         & 76.20          & 72.52                & 69.68                & 66.62                 \\
        EM~                   & 13.20          & 22.08          & 71.43          & 87.71          & 88.62          & 1.60           & 71.47         & 68.49         & 65.66         & 63.43          & \textbf{1.26} & 74.88          & 71.74                & 66.90                & 63.40                 \\
        TAP~                  & 22.51          & 92.16          & 90.53          & 89.55          & 88.02          & 13.75          & 70.03         & 66.91         & 64.30         & 62.39          & 9.10          & 75.14          & 70.56                & 67.64                & 63.56                 \\
        NTGA                  & 16.27          & 41.53          & 85.13          & 89.41          & 88.96          & 3.22           & 65.74         & 66.53         & 64.80         & 62.44          & 8.42          & 63.28          & 66.96                & 65.98                & 63.06                 \\
        SC~                   & \textbf{11.63} & 91.71          & 90.42          & 86.84          & 87.26          & \textbf{ 1.51} & 70.62         & 67.95         & 65.81         & 63.30          &  11.0             &    75.06            & 71.26 & 67.14 & 62.58  \\
        REM~                  & 15.18          & 14.28          & 25.41          & 30.85          & 48.16          & 1.89        & 4.47          & 7.03          & 17.55         & 27.10          & 13.74         & 21.58          & 29.40                & 35.76                & 41.66                 \\ 
        \midrule
        \textbf{SEM}          & 13.0 & \textbf{12.16} & \textbf{11.49} & \textbf{20.91} & \textbf{31.92} & 1.95           & \textbf{3.26} & \textbf{4.35} & \textbf{9.07} & \textbf{20.25} & 4.1           & \textbf{10.34} & \textbf{13.76}       & \textbf{23.58}       & \textbf{37.82}        \\
        \bottomrule
    \end{tabular}
    }
    \label{tab:diff_rad}
\end{table*}

\begin{figure*}[htbp]
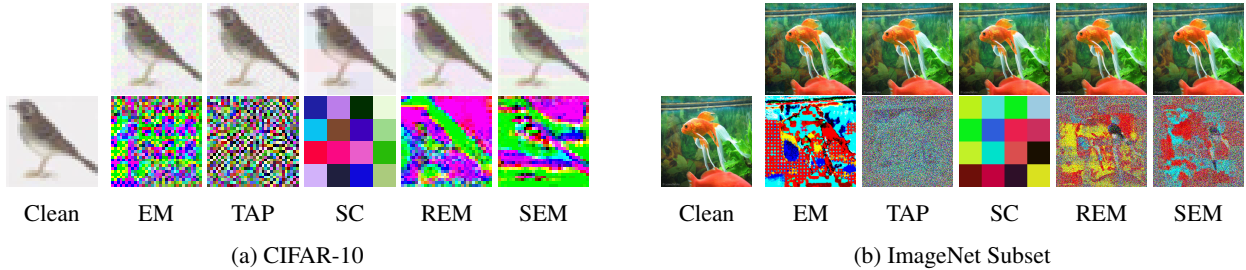

    \hspace{10pt}
    \addfigs{cifar10}{rad-8}{5132}{CIFAR-10}{0.45}
    \hspace{10pt}
    \addfigs{img-three}{rad-8}{1596}{ImageNet Subset}{0.45}
    \caption{
     Visualization of various noise and crafted examples for CIFAR-10 and ImageNet Subset datasets. The noise includes EM (Error-Minimizing noise), TAP (Targeted Adversarial Poisoning noise), NTGA (Neural Tangent Generalization Attack noise), SC (Shortcut noise), REM (Robust Error-Minimizing noise), and SEM (our proposed Stable Error-Minimizing noise).
    }
    \label{fig:vis-example-demo}
\end{figure*}

\section{Experiments}

\subsection{Experimental Setup}
\label{sec.setup}
\header{Dataset.}~We conducted extensive experiments on three widely recognized vision benchmark datasets, including CIFAR-10, CIFAR-100 \cite{krizhevsky2009learning}, and a subset of the ImageNet dataset \cite{russakovsky2015imagenet}, which comprises the first 100 classes from the original ImageNet. In line with \cite{fuRobustUnlearnableExamples2022}, we utilized random cropping and horizontal flipping for data augmentations. 

\header{Model Architecture and Adversarial Training.} We evaluate the proposed method and the baselines across various vision tasks using five popular network architectures: VGG-16~\cite{simonyan2014very}, ResNet18/50~\cite{he2016deep}, DenseNet121~\cite{huang2017densely}, and Wide ResNet-34-10~\cite{zagoruyko2016wide}.
ResNet-18 is selected as the default target model in our experiments. 
By default, the defensive noise radius $\rho_u$ is set at $8/255$, and the adversarial training radius $\rho_a$ is set at $4/255$ for each dataset.
Besides, the setting for the adversarial training radius follows the guidelines of REM~\cite{fuRobustUnlearnableExamples2022}, ensuring $\rho_a \leq \rho_u$.
 
\header{Baselines.} We compare our stable error-minimizing noise {(\textbf{SEM})} with existing SoTA methods, including targeted adversarial poisoning noise {(\textbf{TAP})} \cite{fowl2021preventing}, error-minimizing noise {(\textbf{EM})} \cite{huang2021exploring}, robust error-minimizing noise {(\textbf{REM})} \cite{fuRobustUnlearnableExamples2022}, synthesized shortcut noise (\textbf{SC})~\cite{yuAvailabilityAttacksCreate2022}, and neural tangent generalization attack noise (\textbf{NTGA}) \cite{pmlr-v139-yuan21b}.

\subsection{Performance Analysis}

\begin{table}
    \renewcommand{\arraystretch}{1.1}
    \centering
    \caption{Test accuracy (\%) of different types of models on CIFAR-10 and CIFAR-100 datasets. The adversarial training perturbation radius is set as $4/255$. The defensive perturbation radius $\rho_u$ of every type of defensive noise is set as $8/255$. }
    \label{tab:diff_arch}
    \resizebox{\linewidth}{!}{
        \begin{tabular}{cccccccc} 
            \toprule
            Dataset                    & Model     & Clean & EM    & TAP            & NTGA  & REM~           & SEM~            \\ 
            \hline
            \multirow{5}{*}{CIFAR-10}  & VGG-16    & 87.51 & 86.48 & 86.27          & 86.65 & 65.23          & \textbf{44.37}  \\
                                       & RN-18     & 89.51 & 88.62 & 88.02          & 88.96 & 48.16          & \textbf{31.92}  \\
                                       & RN-50     & 89.79 & 89.28 & 88.45          & 88.79 & 40.65          & \textbf{28.89}  \\
                                       & DN-121    & 83.27 & 82.44 & 81.72          & 80.73 & 81.48          & \textbf{77.85}  \\
                                       & WRN-34-10 & 91.21 & 90.05 & 90.23          & 89.95 & 48.39          & \textbf{31.42}  \\ 
            \hline
            \multirow{5}{*}{CIFAR-100} & VGG-16    & 57.14 & 56.94 & 55.24          & 55.81 & \textbf{48.85} & 57.11
            \\
                                       & RN-18     & 63.43 & 64.17 & 62.39          & 62.44 & 27.10          & \textbf{20.25}  \\
                                       & RN-50     & 66.93 & 66.43 & 64.44          & 64.91 & 26.03          & \textbf{20.99}  \\
                                       & DN-121    & 53.73 & 53.52 & \textbf{52.93} & 52.40 & 54.48          & 55.36           \\
                                       & WRN-34-10 & 68.64 & 68.27 & 65.80          & 67.41 & 25.04          & \textbf{18.90}  \\
            \bottomrule
            \end{tabular}
    }
        
    \end{table}

\begin{table}[thbp]
    \centering
    \renewcommand\arraystretch{0.8}
    \caption{Test accuracy (\%) of different types of models trained on CIFAR-10 and CIFAR-100 that are processed by different low-pass filters.
    The defensive perturbation radius $\rho_u$ of every defensive noise is set as $16/255$.
    }
    \label{tab:diff_lp_filters}
    \resizebox{1.0\linewidth}{!}{
        \begin{tabular}{c c c c c c c c}
            \toprule
            \multirow{2}{1.2cm}{\centering Dataset} & \multirow{2}{0.8cm}{\centering Filter} & \multirow{2}{0.8cm}{\centering Clean} & \multirow{2}{0.8cm}{\centering EM} & \multirow{2}{0.8cm}{\centering TAP} & \multirow{2}{0.8cm}{\centering NTGA} & \multirow{2}{0.8cm}{\centering REM} & \multirow{2}{0.8cm}{\centering SEM} \\
            \\
            \midrule
            \multirow{3}{1.2cm}{\centering CIFAR-10} &
            \multirow{1}{1.2cm}{\centering Mean} &
            84.25 & 34.87 & 82.53 & 40.26 & 28.60 & \textbf{23.93} \\
            \cmidrule(lr){2-8}
            &
            \multirow{1}{1.2cm}{\centering Median} &
            87.04 & 31.86 & 85.10 & 30.87 & 27.36 & \textbf{22.00} \\
            \cmidrule(lr){2-8}
            &
            \multirow{1}{1.2cm}{\centering Gaussian} &
            86.78 & 29.71 & 85.44 & 41.85 & 28.70 & \textbf{21.77} \\
            \midrule
            \multirow{3}{1.2cm}{\centering CIFAR-100} &
            \multirow{1}{1.2cm}{\centering Mean} &
            52.42 & 53.07 & 51.30 & 26.49 & 13.89 & \textbf{8.81} \\
            \cmidrule(lr){2-8}
            &
            \multirow{1}{1.2cm}{\centering Median} &
            57.69 & 56.35 & 55.22 & 18.14 & 14.08 & \textbf{9.02} \\
            \cmidrule(lr){2-8}
            &
            \multirow{1}{1.2cm}{\centering Gaussian} &
            56.64 & 56.49 & 55.19 & 29.05 & 13.74 & \textbf{8.10} \\
            \bottomrule
            \end{tabular}
    }
\end{table}

\header{Different Radii of Adversarial Training. }Our initial evaluation focuses on the protection performance of various unlearnable examples under different adversarial training radii. The defensive perturbation radius is set to $\rho_u=8/255$, with ResNet-18 serving as both the source and target models. As shown in Tab. \ref{tab: comparison}, SC performs best under standard training settings ($\rho_a = 0$). However, with an increase in the adversarial training perturbation radius, the test accuracy of SC also rises significantly. Likewise, the other two baselines, EM and TAP, experience a significant increase in test accuracy with intensified adversarial training. In contrast, REM and SEM, designed for better robustness against adversarial perturbations, do not significantly increase test accuracy even with larger radii. When compared to REM, the results indicate that our SEM consistently outperforms this baseline in all adversarial training settings ($\rho_a \in [\frac{1}{255}, \frac{4}{255}]$), demonstrating the superior effectiveness of our proposed method in generating robust unlearnable examples.

\header{Different Architectures.} To investigate the transferability of different noises across neural network architectures, we chose ResNet-18 as the source model. We then examined the impact of its generated noise on various target models, including VGG-16, ResNet-50, DenseNet-121, and Wide ResNet-34-10. 
All transferability experiments were conducted using CIFAR-10 and CIFAR-100 datasets, as shown in Tab. \ref{tab:diff_arch}. 
The results show that unlearnable examples generated by the SEM method are generally more effective across architecture settings than REM, indicating higher transferability. 

\subsection{Sensitivity Analysis}
\label{app. sen}
\header{Resistance to Low-Pass Filtering.} We next analyze how various defensive noises withstand low-pass filtering. This approach is motivated by the possibility that filtering the image could eliminate the added defensive noise. As per \cite{fuRobustUnlearnableExamples2022}, we employed three low-pass filters: Mean, Median, and Gaussian \cite{young1995recursive}, each with a $3\times3$ window size. We set the adversarial training perturbation radius to $2/255$. The results in Tab. \ref{tab:diff_lp_filters} show that the test accuracy of the models trained on the protected data increases after applying low-pass filters, implying that such filtering partially degrades the added defensive noise. 
Nevertheless, SEM outperforms under both scenarios, with and without applying low-pass filtering in adversarial training.

\header{Resistance to Early Stopping and Partial Poisoning.} One might wonder how the early stopping technique influences the protection performance of our unlearnable examples.
To address this, we conducted experiments on CIFAR-10 and CIFAR-100, varying the early-stopping patience steps. 
We designated 10\% of the unlearnable examples for the validation set, with early-stopping patience $S_{es}$ set to range from 3k to 20k steps.
For examining partial poisoning, we varied the proportion of clean images in the validation set from 10\% to 70\%.
Results are detailed in Tab. \ref{tab: early_stopping}.
Observations from full poisoning (clean ratio equal to 0\%) show that setting $S_{es}$ to 1500 increases test accuracy by 4\%, implying minimal impact of early stopping on unlearnable examples.
However, this effect is not substantial and necessitates searching the early-stopping patience hyper-parameter for mitigation.
Increasing the clean ratio in the validation set further restores test accuracy to 57.48\%, yet it remains below the original level of 89.51\%. 
These results underscore the resistance of our unlearnable examples to early stopping.

\begin{table}[htbp]
    \centering
    \caption{Effect of early stop with different patience steps $S_{es}$. }
    \label{tab: early_stopping}
    \resizebox{\linewidth}{!}{
        \begin{tabular}{c|c|cccc} 
            \toprule
            
            Before                 & Clean Ratio (\%) & $S_{es}=3000$    & $25000$ & $10000$ & $20000$  \\
            \midrule
            \multirow{4}{*}{31.92} & 0                & \textbf{35.80} & 29.34   & 31.30   & 30.87    \\
            & 10 & 31.89 & \textbf{47.22} & 32.02 & 32.02 \\
            & 30 & 31.89 & \textbf{32.02} & 31.67 & 32.02 \\
            & 50 & \textbf{57.48} & 47.22 & 57.48 & 57.48 \\
            & 70 & \textbf{57.48} & 47.22 & 57.48 & 31.89 \\
            \bottomrule
            \end{tabular}
    }
\end{table}

\begin{figure}[htpb]
\begin{minipage}[t]{0.48\linewidth}
  \centering
  \includegraphics[width=\linewidth]{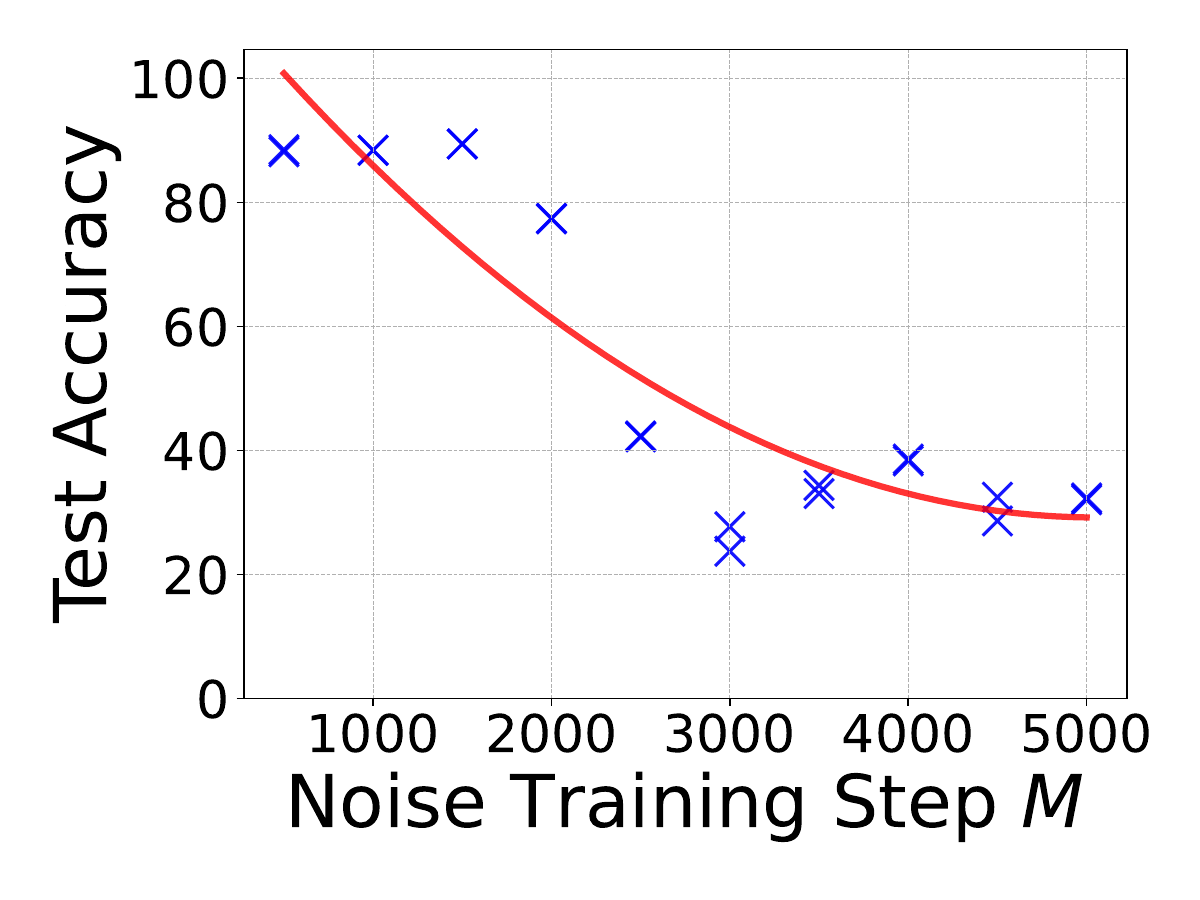}
\end{minipage}
\hfill 
\begin{minipage}[t]{0.48\linewidth}
  \centering
  \includegraphics[width=\linewidth]{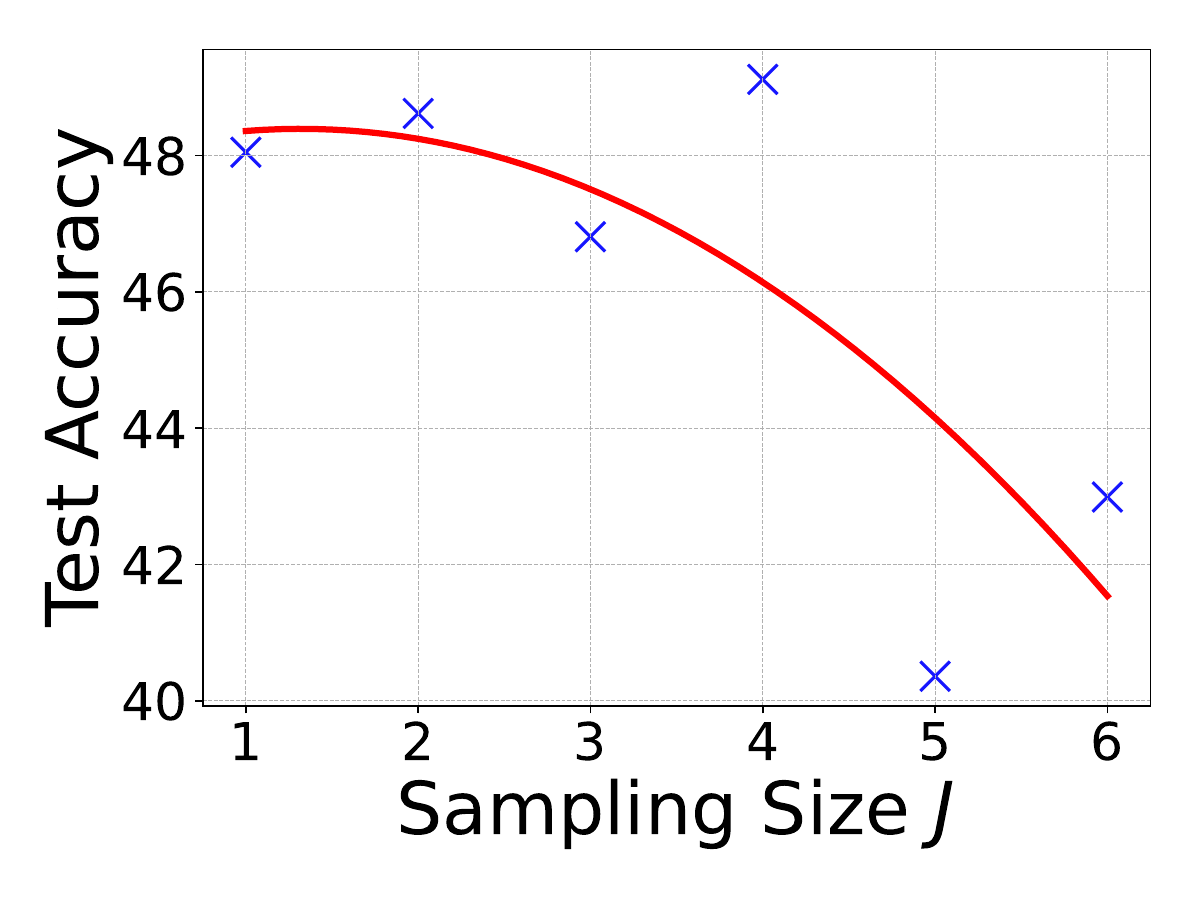}
\end{minipage}
\caption{Effect of noise training steps and sampling step size on the testing accuracy of the trained model. }
  \label{fig: noise_step_num_vs_test_acc}
\end{figure}

\header{Effect of Training Step $M$ and Sampling Size $J$.} We evaluated the impact of the training step $M$ and sampling size $J$ during noise generation, considering $M$ values from 500 to 5000 and $J$ ranging from 1 to 6. Results, as illustrated in Fig. \ref{fig: noise_step_num_vs_test_acc}, indicate that an increase in training steps enhances performance. This improvement could be attributed to the noise generator learning more robust noise patterns over extended training steps. Regarding sampling size $J$, an increase in $J$ was found to reduce the testing accuracy by nearly 8\%, underscoring its significance in data protection. 

\begin{table}
    \centering
    \caption{Effect of different radii of random perturbation.}
    \label{tab: diff_radius_random}
\resizebox{\linewidth}{!}{
    \begin{tabular}{cccccc} 
        \toprule
        Adv. Train. $\rho_a$ & $\rho_r = 0$ & $1/255$        & $2/255$        & $3/255$        & $4/255$         \\ 
        \midrule
        $0$                  & 15.87        & 11.41          & \textbf{10.42} & 16.83          & 10.52           \\
        $1/255$              & 16.59        & \textbf{12.11} & 15.80          & 17.92          & 12.16           \\
        $2/255$              & 64.24        & 13.44          & \textbf{11.49} & 19.13          & 21.68           \\
        $3/255$              & 89.18        & 28.6           & 21.88          & \textbf{20.91} & 23.43           \\
        $4/255$              & 89.30        & 67.52          & 35.63          & 62.09          & \textbf{31.92}  \\
        \bottomrule
        \end{tabular}
}
\end{table}

\header{Effect of Random Perturbation Radius.} The radius of random perturbation represents a crucial hyper-parameter in our approach. To explore its correlation with protection performance, we conducted experiments on CIFAR-10 and CIFAR-100 using various radii of random perturbation and adversarial training. Results are detailed in Tab. \ref{tab: diff_radius_random}. The results indicate that the best data protection performance is generally achieved when the random perturbation radius $\rho_r$ is set equal to the adversarial training radius $\rho_a$, across varying adversarial training radii. Nevertheless, when these two radii are mismatched, the protection performance drops slightly, indicating the importance of knowledge about adversarial training radius for protection effectiveness. 

\subsection{Ablation Study}
We conducted an ablation study to understand the impact of various components in our method, specifically focusing on the adversarial noise used to update the noise generator and the random noise used to update the defensive noise. 
Results, as shown in Tab. \ref{tab: ablation}, reveal that the adversarial noise $\delta^a$, used in updating the noise generator $\theta$, is a critical factor for protection performance. Its removal results in an approximate 60\% decrease in accuracy. 
Additionally, the elimination of random perturbation in updating the defensive noise leads to a reduction in protection performance by approximately 8\%. 
Furthermore, using purely adversarial noise to update the defensive noise results in deteriorated protection performance, highlighting the adverse effects of adversarial perturbations.

\begin{table}
    \centering
    \caption{Ablation study of the proposed method.}
    \label{tab: ablation}
\resizebox{\linewidth}{!}{
    \begin{tabular}{cccccc} 
        \toprule
        \multicolumn{2}{c}{$\delta^{u*}$} & \multicolumn{2}{c}{$\theta^*$} & \multirow{2}{*}{Test Acc. (\%)} & \multirow{2}{*}{Acc. Increase (\%)} \\ 
        \cmidrule(lr){1-2}\cmidrule(lr){2-4}
        \multicolumn{1}{c}{$\delta^{r}$} & \multicolumn{1}{c}{$\delta^{a}$} & \multicolumn{1}{c}{$\delta^{r}$} & \multicolumn{1}{c}{$\delta^{a}$}  \\ 
        \midrule
        $\surd$ & & & $\surd$ & \textbf{30.03} & - \\
        \midrule
        & & & $\surd$ & 37.91 & +7.88 \\
        & $\surd$ & & $\surd$ & 46.72 & +16.69 \\
        $\surd$ & & & & 89.22 & +59.19 \\
        & & & & 89.15 & +59.13 \\
        \bottomrule
    \end{tabular}
}
\end{table}

\subsection{Case study: Face Recognition}
To evaluate our proposed method's effectiveness on real-world face recognition, we conducted experiments using the Facescrub dataset~\cite{ng2014data}. Specifically, we randomly selected ten classes, with each class comprising 120 images. We allocated 15\% of the data as the testing set, resulting in 1020 images for training and 180 for testing. 
Results presented in Tab. \ref{tab: face} indicate that robust methods, namely REM and SEM, outperform other non-robust methods in terms of data protection under adversarial training. 
In particular, the test accuracy of models trained on SEM-protected data fell to around 9\%, across various adversarial training radii, marking a significant drop in accuracy by approximately 40\% to 50\%. 
Additionally, we observed from Fig. \ref{fig: face-vis} that robust methods result in more vividly protected images. 
For instance, SC creates mosaic-like images, while EM and TAP substantially reduce luminance, impairing facial recognition. Conversely, REM and SEM perturbations concentrate on edges, maintaining visual similarity.

\begin{table}
\centering
\caption{Evaluation on a real-world face recognition dataset. }
\label{tab: face}
\resizebox{\linewidth}{!}{
    \begin{tabular}{cccccc}
        \toprule
        Methods & \(\rho_a=0\) & \(1/255\) & \(2/255\) & \(3/255\) & \(4/255\) \\
        \midrule
        Clean Data & 68.89 & 63.33 & 61.11 & 53.33 & 50.56\\
        EM~ & \textbf{8.89} & 18.06 & 18.61 & 16.39 & 19.72\\
       TAP~ & 9.72 & 20.74 & 26.48 & 32.22 & 36.39\\
        SC~ & 66.11 & 66.67 & 64.72 & 57.5 & 59.44\\
        REM~ & 10.89 & 7.89 & 11.67 & 12.22 & 11.48\\
        \midrule
        \textbf{SEM}& 10.56 & \textbf{7.04 }& \textbf{9.44} & \textbf{10.28} & \textbf{11.11}\\
        \bottomrule
    \end{tabular}
}
\end{table}

\begin{figure}[thbp]
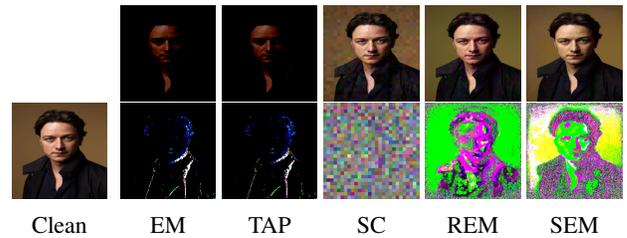

    \hspace{10pt}
    \centering
    \addfigs{face-10}{rad-8}{72}{}{1.0}
    \caption{
    Visualization of different types of defensive noise and crafted unlearnable examples on the Facescrub subset.
    }
    \label{fig: face-vis}
\end{figure}

\section{Related Works}
\header{Data Poisoning.} Data poisoning attacks aim to manipulate the performance of a machine learning model on clean examples by modifying the training examples. Recent research has shown the vulnerability of both DNNs~\cite{munoz2017towards} and classical machine learning methods, such as SVM \cite{biggio2012poisoning}, to poisoning attacks \cite{shafahi2018poison}. Recent advancements use gradient matching and meta-learning techniques to address the noise crafting problem \cite{geiping2020witches}. Backdoor attacks are kinds of data poisoning that try to inject falsely labeled training samples with a concealed trigger \cite{gu2017badnets}. Unlearnable examples, regarded as a clean-label and triggerless backdoor approach~\cite{gan2022triggerless, souri2022sleeper, zeng2023narcissus, xian2023understanding, li2022untargeted, liu2022friendly}, perturb features to impair the model's generalization ability.

\header{Unlearnable Example.} \citet{huang2020unlearnable} introduced error-minimizing noise to create ``unlearnable examples'', causing deep learning models to learn irrelevant features and generalize ineffectively.
Unlearnable examples have been adapted for data protection in various domains and applications \cite{liu2023unlearnable,liu2023securing, liu2023toward, liu2023graphcloak, sun2022coprotector,zhang2023unlearnable, li2023unganable, he2022indiscriminate, zhao2022clpa,salman2023raising, guo2023domain}. 
Despite its effectiveness, such conferred unlearnability is found fragile to adversarial training \cite{madry2018towards} and data transformations \cite{athalye2018synthesizing}. 
Recent studies \cite{tao2021better} show that applying adversarial training with a suitable radius can counteract the added noise, thereby restoring model performance and making the data ``learnable'' again.
Addressing this, \citet{fuRobustUnlearnableExamples2022} proposed a min-min-max optimization approach to create more robust unlearnable examples by learning a noise generator and developing robust error-minimizing noise. 
However, this method faces challenges with inefficiency and suboptimal protection performance, mainly due to the instability of defensive noise.
To overcome these limitations, our approach offers a more efficient and effective method for training enhanced defensive noise.

\section{Conclusion}
In this work, we introduced Stable Error-Minimizing (SEM), a novel method for generating unlearnable examples that achieve better protection performance and generation efficiency under adversarial training. Our research uncovers an intriguing phenomenon: the effectiveness of unlearnable examples in protecting data is predominantly derived from the robustness of the surrogate model. Motivated by this, SEM employs a surrogate model to create robust error-minimizing noise, effectively countering random perturbations. Extensive experimental evaluations confirm that SEM outperforms previous SoTA in terms of both effectiveness and efficiency. 

\section*{Acknowledgements}
This work was jointly conducted at Samsung Research America and Lehigh University, with partial support from the National Science Foundation under Grant CRII-2246067 and CCF-2319242. We thank Jeff Heflin and Maryann DiEdwardo for their valuable early-stage feedback on the manuscript; Shaopeng Fu for his insights regarding the REM codebase; Weiran Huang for engaging discussions and unique perspectives. We also thank our AAAI reviewers and meta-reviewers for their constructive suggestions. 

\bibliography{aaai24}

\begin{thebibliography}{50}
\providecommand{\natexlab}[1]{#1}

\bibitem[{Athalye et~al.(2018)Athalye, Engstrom, Ilyas, and Kwok}]{athalye2018synthesizing}
Athalye, A.; Engstrom, L.; Ilyas, A.; and Kwok, K. 2018.
\newblock Synthesizing robust adversarial examples.
\newblock In \emph{International conference on machine learning}, 284--293. PMLR.

\bibitem[{Biggio, Nelson, and Laskov(2012)}]{biggio2012poisoning}
Biggio, B.; Nelson, B.; and Laskov, P. 2012.
\newblock Poisoning attacks against support vector machines.
\newblock In \emph{ICML}.

\bibitem[{Cao et~al.(2023)Cao, Li, Liu, Yan, Dai, Yu, and Sun}]{cao2023comprehensive}
Cao, Y.; Li, S.; Liu, Y.; Yan, Z.; Dai, Y.; Yu, P.~S.; and Sun, L. 2023.
\newblock A comprehensive survey of ai-generated content (aigc): A history of generative ai from gan to chatgpt.
\newblock \emph{arXiv preprint arXiv:2303.04226}.

\bibitem[{Devlin et~al.(2019)Devlin, Chang, Lee, and Toutanova}]{devlin2018bert}
Devlin, J.; Chang, M.-W.; Lee, K.; and Toutanova, K. 2019.
\newblock Bert: Pre-training of deep bidirectional transformers for language understanding.

\bibitem[{Dosovitskiy et~al.(2020)Dosovitskiy, Beyer, Kolesnikov, Weissenborn, Zhai, Unterthiner, Dehghani, Minderer, Heigold, Gelly et~al.}]{dosovitskiy2020image}
Dosovitskiy, A.; Beyer, L.; Kolesnikov, A.; Weissenborn, D.; Zhai, X.; Unterthiner, T.; Dehghani, M.; Minderer, M.; Heigold, G.; Gelly, S.; et~al. 2020.
\newblock An Image is Worth 16x16 Words: Transformers for Image Recognition at Scale.
\newblock In \emph{International Conference on Learning Representations}.

\bibitem[{Feng, Cai, and Zhou(2019)}]{feng2019learning}
Feng, J.; Cai, Q.-Z.; and Zhou, Z.-H. 2019.
\newblock Learning to confuse: generating training time adversarial data with auto-encoder.
\newblock In \emph{NeurIPS}.

\bibitem[{Fowl et~al.(2021{\natexlab{a}})Fowl, Chiang, Goldblum, Geiping, Bansal, Czaja, and Goldstein}]{fowl2021preventing}
Fowl, L.; Chiang, P.-y.; Goldblum, M.; Geiping, J.; Bansal, A.; Czaja, W.; and Goldstein, T. 2021{\natexlab{a}}.
\newblock Preventing unauthorized use of proprietary data: Poisoning for secure dataset release.
\newblock \emph{arXiv preprint arXiv:2103.02683}.

\bibitem[{Fowl et~al.(2021{\natexlab{b}})Fowl, Goldblum, Chiang, Geiping, Czaja, and Goldstein}]{fowl2021adversarial}
Fowl, L.; Goldblum, M.; Chiang, P.-y.; Geiping, J.; Czaja, W.; and Goldstein, T. 2021{\natexlab{b}}.
\newblock Adversarial Examples Make Strong Poisons.
\newblock In \emph{NeurIPS}.

\bibitem[{Fu et~al.(2022)Fu, He, Liu, Shen, and Tao}]{fuRobustUnlearnableExamples2022}
Fu, S.; He, F.; Liu, Y.; Shen, L.; and Tao, D. 2022.
\newblock Robust {{Unlearnable Examples}}: {{Protecting Data Privacy Against Adversarial Learning}}.
\newblock In \emph{International {{Conference}} on {{Learning Representations}}}.

\bibitem[{Gan et~al.(2022)Gan, Li, Zhang, Li, Meng, Wu, Yang, Guo, and Fan}]{gan2022triggerless}
Gan, L.; Li, J.; Zhang, T.; Li, X.; Meng, Y.; Wu, F.; Yang, Y.; Guo, S.; and Fan, C. 2022.
\newblock Triggerless Backdoor Attack for NLP Tasks with Clean Labels.
\newblock In \emph{Proceedings of the 2022 Conference of the North American Chapter of the Association for Computational Linguistics: Human Language Technologies}, 2942--2952.

\bibitem[{Geiping et~al.(2020)Geiping, Fowl, Huang, Czaja, Taylor, Moeller, and Goldstein}]{geiping2020witches}
Geiping, J.; Fowl, L.~H.; Huang, W.~R.; Czaja, W.; Taylor, G.; Moeller, M.; and Goldstein, T. 2020.
\newblock Witches' Brew: Industrial Scale Data Poisoning via Gradient Matching.
\newblock In \emph{International Conference on Learning Representations}.

\bibitem[{Gu, Dolan-Gavitt, and Garg(2017)}]{gu2017badnets}
Gu, T.; Dolan-Gavitt, B.; and Garg, S. 2017.
\newblock Badnets: Identifying vulnerabilities in the machine learning model supply chain.
\newblock \emph{arXiv preprint arXiv:1708.06733}.

\bibitem[{Guo et~al.(2023)Guo, Li, Wang, Xia, Huang, Liu, and Li}]{guo2023domain}
Guo, J.; Li, Y.; Wang, L.; Xia, S.-T.; Huang, H.; Liu, C.; and Li, B. 2023.
\newblock Domain Watermark: Effective and Harmless Dataset Copyright Protection is Closed at Hand.
\newblock In \emph{Thirty-seventh Conference on Neural Information Processing Systems}.

\bibitem[{He, Zha, and Katabi(2022)}]{he2022indiscriminate}
He, H.; Zha, K.; and Katabi, D. 2022.
\newblock Indiscriminate Poisoning Attacks on Unsupervised Contrastive Learning.
\newblock In \emph{The Eleventh International Conference on Learning Representations}.

\bibitem[{He et~al.(2016)He, Zhang, Ren, and Sun}]{he2016deep}
He, K.; Zhang, X.; Ren, S.; and Sun, J. 2016.
\newblock Deep residual learning for image recognition.
\newblock In \emph{CVPR}.

\bibitem[{Hill and Krolik(2019)}]{hill2019photos}
Hill, K.; and Krolik, A. 2019.
\newblock How photos of your kids are powering surveillance technology.
\newblock \emph{The New York Times}.

\bibitem[{Huang et~al.(2017)Huang, Liu, Van Der~Maaten, and Weinberger}]{huang2017densely}
Huang, G.; Liu, Z.; Van Der~Maaten, L.; and Weinberger, K.~Q. 2017.
\newblock Densely connected convolutional networks.
\newblock In \emph{CVPR}.

\bibitem[{Huang et~al.(2020)Huang, Ma, Erfani, Bailey, and Wang}]{huang2020unlearnable}
Huang, H.; Ma, X.; Erfani, S.~M.; Bailey, J.; and Wang, Y. 2020.
\newblock Unlearnable Examples: Making Personal Data Unexploitable.
\newblock In \emph{International Conference on Learning Representations}.

\bibitem[{Huang et~al.(2021)Huang, Wang, Erfani, Gu, Bailey, and Ma}]{huang2021exploring}
Huang, H.; Wang, Y.; Erfani, S.~M.; Gu, Q.; Bailey, J.; and Ma, X. 2021.
\newblock Exploring Architectural Ingredients of Adversarially Robust Deep Neural Networks.
\newblock In \emph{NeurIPS}.

\bibitem[{Kipf and Welling(2016)}]{kipf2016semi}
Kipf, T.~N.; and Welling, M. 2016.
\newblock Semi-supervised classification with graph convolutional networks.
\newblock \emph{arXiv preprint arXiv:1609.02907}.

\bibitem[{Krizhevsky et~al.(2009)}]{krizhevsky2009learning}
Krizhevsky, A.; et~al. 2009.
\newblock Learning multiple layers of features from tiny images.

\bibitem[{Li et~al.(2022)Li, Bai, Jiang, Yang, Xia, and Li}]{li2022untargeted}
Li, Y.; Bai, Y.; Jiang, Y.; Yang, Y.; Xia, S.-T.; and Li, B. 2022.
\newblock Untargeted backdoor watermark: Towards harmless and stealthy dataset copyright protection.
\newblock \emph{Advances in Neural Information Processing Systems}, 35: 13238--13250.

\bibitem[{Li et~al.(2023)Li, Yu, Salem, Backes, Fritz, and Zhang}]{li2023unganable}
Li, Z.; Yu, N.; Salem, A.; Backes, M.; Fritz, M.; and Zhang, Y. 2023.
\newblock $\{$UnGANable$\}$: Defending Against $\{$GAN-based$\}$ Face Manipulation.
\newblock In \emph{32nd USENIX Security Symposium (USENIX Security 23)}, 7213--7230.

\bibitem[{Liu, Yang, and Mirzasoleiman(2022)}]{liu2022friendly}
Liu, T.~Y.; Yang, Y.; and Mirzasoleiman, B. 2022.
\newblock Friendly noise against adversarial noise: a powerful defense against data poisoning attack.
\newblock \emph{Advances in Neural Information Processing Systems}, 35: 11947--11959.

\bibitem[{Liu et~al.(2023{\natexlab{a}})Liu, Fan, Chen, Zhou, and Sun}]{liu2023graphcloak}
Liu, Y.; Fan, C.; Chen, X.; Zhou, P.; and Sun, L. 2023{\natexlab{a}}.
\newblock GraphCloak: Safeguarding Task-specific Knowledge within Graph-structured Data from Unauthorized Exploitation.
\newblock \emph{arXiv preprint arXiv:2310.07100}.

\bibitem[{Liu et~al.(2023{\natexlab{b}})Liu, Fan, Dai, Chen, Zhou, and Sun}]{liu2023toward}
Liu, Y.; Fan, C.; Dai, Y.; Chen, X.; Zhou, P.; and Sun, L. 2023{\natexlab{b}}.
\newblock Toward Robust Imperceptible Perturbation against Unauthorized Text-to-image Diffusion-based Synthesis.
\newblock \emph{arXiv preprint arXiv:2311.13127}.

\bibitem[{Liu et~al.(2023{\natexlab{c}})Liu, Fan, Zhou, and Sun}]{liu2023unlearnable}
Liu, Y.; Fan, C.; Zhou, P.; and Sun, L. 2023{\natexlab{c}}.
\newblock Unlearnable Graph: Protecting Graphs from Unauthorized Exploitation.
\newblock \emph{arXiv preprint arXiv:2303.02568}.

\bibitem[{Liu et~al.(2023{\natexlab{d}})Liu, Ye, Zhang, and Sun}]{liu2023securing}
Liu, Y.; Ye, H.; Zhang, K.; and Sun, L. 2023{\natexlab{d}}.
\newblock Securing Biomedical Images from Unauthorized Training with Anti-Learning Perturbation.
\newblock \emph{arXiv preprint arXiv:2303.02559}.

\bibitem[{Madry et~al.(2018)Madry, Makelov, Schmidt, Tsipras, and Vladu}]{madry2018towards}
Madry, A.; Makelov, A.; Schmidt, L.; Tsipras, D.; and Vladu, A. 2018.
\newblock Towards Deep Learning Models Resistant to Adversarial Attacks.
\newblock In \emph{ICLR}.

\bibitem[{Mu{\~n}oz-Gonz{\'a}lez et~al.(2017)Mu{\~n}oz-Gonz{\'a}lez, Biggio, Demontis, Paudice, Wongrassamee, Lupu, and Roli}]{munoz2017towards}
Mu{\~n}oz-Gonz{\'a}lez, L.; Biggio, B.; Demontis, A.; Paudice, A.; Wongrassamee, V.; Lupu, E.~C.; and Roli, F. 2017.
\newblock Towards poisoning of deep learning algorithms with back-gradient optimization.
\newblock In \emph{ACM Workshop on Artificial Intelligence and Security}.

\bibitem[{Ng and Winkler(2014)}]{ng2014data}
Ng, H.-W.; and Winkler, S. 2014.
\newblock A data-driven approach to cleaning large face datasets.
\newblock In \emph{2014 IEEE International Conference on Image Processing (ICIP)}, 343--347.

\bibitem[{Russakovsky et~al.(2015)Russakovsky, Deng, Su, Krause, Satheesh, Ma, Huang, Karpathy, Khosla, Bernstein et~al.}]{russakovsky2015imagenet}
Russakovsky, O.; Deng, J.; Su, H.; Krause, J.; Satheesh, S.; Ma, S.; Huang, Z.; Karpathy, A.; Khosla, A.; Bernstein, M.; et~al. 2015.
\newblock Imagenet large scale visual recognition challenge.
\newblock \emph{International journal of computer vision}, 115(3): 211--252.

\bibitem[{Salman et~al.(2023)Salman, Khaddaj, Leclerc, Ilyas, and Madry}]{salman2023raising}
Salman, H.; Khaddaj, A.; Leclerc, G.; Ilyas, A.; and Madry, A. 2023.
\newblock Raising the cost of malicious ai-powered image editing.
\newblock \emph{arXiv preprint arXiv:2302.06588}.

\bibitem[{Shafahi et~al.(2018)Shafahi, Huang, Najibi, Suciu, Studer, Dumitras, and Goldstein}]{shafahi2018poison}
Shafahi, A.; Huang, W.~R.; Najibi, M.; Suciu, O.; Studer, C.; Dumitras, T.; and Goldstein, T. 2018.
\newblock Poison frogs! targeted clean-label poisoning attacks on neural networks.
\newblock In \emph{NeurIPS}.

\bibitem[{Simonyan and Zisserman(2015)}]{simonyan2014very}
Simonyan, K.; and Zisserman, A. 2015.
\newblock Very deep convolutional networks for large-scale image recognition.
\newblock In \emph{ICLR}.

\bibitem[{Souri et~al.(2022)Souri, Fowl, Chellappa, Goldblum, and Goldstein}]{souri2022sleeper}
Souri, H.; Fowl, L.; Chellappa, R.; Goldblum, M.; and Goldstein, T. 2022.
\newblock Sleeper agent: Scalable hidden trigger backdoors for neural networks trained from scratch.
\newblock \emph{Advances in Neural Information Processing Systems}, 35: 19165--19178.

\bibitem[{Sun et~al.(2022{\natexlab{a}})Sun, Dou, Yang, Zhang, Wang, Philip, He, and Li}]{sun2022adversarial}
Sun, L.; Dou, Y.; Yang, C.; Zhang, K.; Wang, J.; Philip, S.~Y.; He, L.; and Li, B. 2022{\natexlab{a}}.
\newblock Adversarial attack and defense on graph data: A survey.
\newblock \emph{IEEE Transactions on Knowledge and Data Engineering}.

\bibitem[{Sun et~al.(2022{\natexlab{b}})Sun, Du, Song, Ni, and Li}]{sun2022coprotector}
Sun, Z.; Du, X.; Song, F.; Ni, M.; and Li, L. 2022{\natexlab{b}}.
\newblock Coprotector: Protect open-source code against unauthorized training usage with data poisoning.
\newblock In \emph{Proceedings of the ACM Web Conference 2022}, 652--660.

\bibitem[{Tao et~al.(2021)Tao, Feng, Yi, Huang, and Chen}]{tao2021better}
Tao, L.; Feng, L.; Yi, J.; Huang, S.-J.; and Chen, S. 2021.
\newblock Better safe than sorry: Preventing delusive adversaries with adversarial training.
\newblock \emph{Advances in Neural Information Processing Systems}, 34: 16209--16225.

\bibitem[{Vaswani et~al.(2017)Vaswani, Shazeer, Parmar, Uszkoreit, Jones, Gomez, Kaiser, and Polosukhin}]{vaswani2017attention}
Vaswani, A.; Shazeer, N.; Parmar, N.; Uszkoreit, J.; Jones, L.; Gomez, A.~N.; Kaiser, {\L}.; and Polosukhin, I. 2017.
\newblock Attention is all you need.
\newblock \emph{Advances in neural information processing systems}, 30.

\bibitem[{Wang et~al.(2018)Wang, Girshick, Gupta, and He}]{wang2018non}
Wang, X.; Girshick, R.; Gupta, A.; and He, K. 2018.
\newblock Non-local neural networks.
\newblock In \emph{Proceedings of the IEEE conference on computer vision and pattern recognition}, 7794--7803.

\bibitem[{Xian et~al.(2023)Xian, Wang, Srinivasa, Kundu, Bi, Hong, and Ding}]{xian2023understanding}
Xian, X.; Wang, G.; Srinivasa, J.; Kundu, A.; Bi, X.; Hong, M.; and Ding, J. 2023.
\newblock Understanding Backdoor Attacks through the Adaptability Hypothesis.
\newblock In \emph{Proc. International Conference on Machine Learning}.

\bibitem[{Young and Van~Vliet(1995)}]{young1995recursive}
Young, I.~T.; and Van~Vliet, L.~J. 1995.
\newblock Recursive implementation of the Gaussian filter.
\newblock \emph{Signal processing}, 44(2): 139--151.

\bibitem[{Yu et~al.(2022)Yu, Zhang, Chen, Yin, and Liu}]{yuAvailabilityAttacksCreate2022}
Yu, D.; Zhang, H.; Chen, W.; Yin, J.; and Liu, T.-Y. 2022.
\newblock Availability {{Attacks Create Shortcuts}}.
\newblock arXiv:2111.00898.

\bibitem[{Yuan and Wu(2021)}]{pmlr-v139-yuan21b}
Yuan, C.-H.; and Wu, S.-H. 2021.
\newblock Neural Tangent Generalization Attacks.
\newblock In \emph{ICML}.

\bibitem[{Zagoruyko and Komodakis(2016)}]{zagoruyko2016wide}
Zagoruyko, S.; and Komodakis, N. 2016.
\newblock Wide Residual Networks.
\newblock In \emph{Procedings of the British Machine Vision Conference 2016}. British Machine Vision Association.

\bibitem[{Zeng et~al.(2023)Zeng, Pan, Just, Lyu, Qiu, and Jia}]{zeng2023narcissus}
Zeng, Y.; Pan, M.; Just, H.~A.; Lyu, L.; Qiu, M.; and Jia, R. 2023.
\newblock Narcissus: A practical clean-label backdoor attack with limited information.
\newblock In \emph{Proceedings of the 2023 ACM SIGSAC Conference on Computer and Communications Security}, 771--785.

\bibitem[{Zhang et~al.(2023)Zhang, Ma, Yi, Sang, Jiang, Wang, and Xu}]{zhang2023unlearnable}
Zhang, J.; Ma, X.; Yi, Q.; Sang, J.; Jiang, Y.-G.; Wang, Y.; and Xu, C. 2023.
\newblock Unlearnable clusters: Towards label-agnostic unlearnable examples.
\newblock In \emph{Proceedings of the IEEE/CVF Conference on Computer Vision and Pattern Recognition}, 3984--3993.

\bibitem[{Zhao and Lao(2022)}]{zhao2022clpa}
Zhao, B.; and Lao, Y. 2022.
\newblock CLPA: Clean-label poisoning availability attacks using generative adversarial nets.
\newblock In \emph{Proceedings of the AAAI Conference on Artificial Intelligence}, volume~36, 9162--9170.

\bibitem[{Zhou et~al.(2023)Zhou, Li, Li, Yu, Liu, Wang, Zhang, Ji, Yan, He et~al.}]{zhou2023comprehensive}
Zhou, C.; Li, Q.; Li, C.; Yu, J.; Liu, Y.; Wang, G.; Zhang, K.; Ji, C.; Yan, Q.; He, L.; et~al. 2023.
\newblock A comprehensive survey on pretrained foundation models: A history from bert to chatgpt.
\newblock \emph{arXiv preprint arXiv:2302.09419}.

\end{thebibliography}
\clearpage
\ifneedappendix
    
\appendix

\section{Justification of Assumption $\rho_u \geq \rho_a$}
\label{app. A}
\begin{theorem}
    Given a classifier $f: \mathcal{X} \rightarrow \mathcal{Y}$, protection radius $\rho_u$, adversarial training radius $\rho_a$, clean dataset $\mathcal{T}^{c}=\{x_i,y_i\}_{i=1 \cdots N}$, protected dataset $\mathcal{T}^{u}=\{(x_i+\delta^{u}_{i},y_i)|\left\| {\delta ^u}_i \right\| _p\le \rho _u\}_{i=1 \cdots N}$. When $\rho_a \geq \rho_u$, conducting adversarial training on the protected dataset $\mathcal{T}^{u}$ is equivalent to minimizing the upper bound of training loss on the original clean dataset $\mathcal{T}^{c}$, which will ensure that it can recover the test accuracy on the original clean dataset and make data protection infeasible. 
\label{the. justi}
\end{theorem}
\begin{proof}
    Denote the natural loss on dataset $\mathcal{T}$ for classifier $f$ as $\mathcal{R} _{\mathrm{nat}}\left( f,\mathcal{T} \right) =\mathbb{E} _{(x,y)\sim \mathcal{T}}\left[ \ell \left( f\left( x \right) ,y \right) \right] $ and the adversarial training loss on dataset $\mathcal{T}$ for classifier $f$ with adversarial radius $r=\epsilon$ as $\mathcal{R} _{\mathrm{adv}}\left( f,\mathcal{T} ,\epsilon \right) =\mathbb{E} _{(x,y)\sim \mathcal{T}}\left[ \max_{\norm{\delta}_p\le \epsilon} \ell \left( f\left( x+\delta \right) ,y \right) \right] $. Conducting adversarial training on the protected dataset $\mathcal{T}^{u}$ with adversarial radius $\rho_a$ has the following objective
    \begin{equation*}
        \begin{aligned}
	\mathcal{R} _{\mathrm{adv}}\left( f,\mathcal{T} ^u,\rho _a \right) &=\mathbb{E} _{(x,y)\sim \mathcal{T} ^u} \max_{||\delta ^a||_p\le \rho _a} \ell \left( f\left( x+\delta ^a \right) ,y \right) \\
&=\mathbb{E} _{(x,y)\sim \mathcal{T} ^c} \max_{\norm{\delta^a}_p\le \rho _a} \ell \left( f\left( x+\delta ^u+\delta ^a \right) ,y \right) 
    \\
	&\ge \mathbb{E} _{(x,y)\sim \mathcal{T} ^c}\left[ \ell \left( f\left( x+\delta ^u-\delta ^u \right) ,y \right) \right]\\
	&=\mathbb{E} _{(x,y)\sim \mathcal{T} ^c}\left[ \ell \left( f\left( x \right) ,y \right) \right]\\
	&=\mathcal{R} _{\mathrm{nat}}\left( f,\mathcal{T} ^c \right)\\
\end{aligned}
    \end{equation*}
The first inequality holds since $\norm{\delta ^u}_p\le \rho _u \leq \rho_a$, thus for each sum term inside the expectation we have,
\begin{equation*}
\begin{aligned}
	\max_{\norm{\delta^a}_p\le \rho _a} \ell \left( f\left( x+\delta ^u+\delta ^a \right) ,y \right) &\ge \ell \left( f\left( x+\delta ^u-\delta ^u \right) ,y \right)\\
	&=\ell \left( f\left( x \right) ,y \right)\\
\end{aligned}
\end{equation*}
\end{proof}

\section{
    Baselines Implementation
}
\label{app. baseline}

\textbf{Error-minimizing noise \citep{huang2020unlearnable}.}
We train the error-minimizing noise generator by solving Eq. \ref{eq. unlearnable example}. The error-minimizing noise is then generated with the trained noise generator.
The ResNet-18 model is used as the source model $f'$.
The PGD method is employed for solving the inner minimization problem in Eq. (\ref{eq. unlearnable example}), where the settings of PGD are presented in Tab. \ref{tab: pgd_params}.
\begin{equation}
\min _\theta \frac{1}{n} \sum_{i=1}^n \min _{\left\|\delta_i\right\| \leq \rho_u} \ell\left(f_\theta^{\prime}\left(x_i+\delta_i\right), y_i\right)
\label{eq. unlearnable example}
\end{equation}

\header{Targeted adversarial poisoning noise \citep{fowl2021adversarial}.}
This type of noise is generated via conducting a targeted adversarial attack on the model trained on clean data, in which the generated adversarial perturbation is used as the adversarial poisoning noise.
Specifically, given a fixed model $f_0$ and a sample $(x,y)$, 
the targeted adversarial attack will generate noise by solving the problem
$\argmax_{\|\delta_u\| \leq \rho_u} \ell(f_0(x + \delta_u), g(y))$,
where $g$ is a permutation function on the label space $\mathcal{Y}$. We set $g(y)=(y+1)\%|Y|$. 
The hyper-parameters for the PGD are given in Tab. \ref{tab: pgd_params}.

\header{Neural tangent generalization attack noise \citep{pmlr-v139-yuan21b}.}
This type of protective noise aims to weaken the generalization ability of the model trained on the modified data. To do this, an ensemble of neural networks is modeled based on a neural tangent kernel (NTK), and the NTGA noise is generated upon the ensemble model. As a result, the NTGA noise enjoys remarkable transferability.
We use the official code of NTGA \footnote{The official code of NTGA is available at \url{https://github.com/lionelmessi6410/ntga}.} to generate this type of noise. Specifically, we employ the FNN model in \citet{pmlr-v139-yuan21b} as the ensemble model. For CIFAR-10 and CIFAR-100, the block size for approximating NTK is set as $4,000$. For the ImageNet subset, the block size is set as $100$. The hyper-parameters for the PGD are given in Tab. \ref{tab: pgd_params}. Please refer \citep{pmlr-v139-yuan21b} accordingly for other settings.

\header{Robust Error-minimizing Noise \citep{fuRobustUnlearnableExamples2022}.}~REM is a robust error-minimizing noise method that aims to improve further the robustness of unlearnable examples against adversarial training and data transformation. The surrogate model is trained by solving the min-min-max optimization in Eq. \ref{eq. rem}, and the noise is obtained via Eq. \ref{eq: final def noise}, 
\begin{equation}
\min _\theta \frac{1}{n} \sum_{i=1}^n \min _{\left\|\delta_i^u\right\| \leq \rho_u} \mathbb{E}_{t \sim T} \max _{\left\|\delta_i^a\right\| \leq \rho_a} \ell\left(f_\theta^{\prime}\left(t\left(x_i+\delta_i^u\right)+\delta_i^a\right), y_i\right)
\label{eq. rem}
\end{equation}

\header{Linearly-separable Synthesized Shortcut Noise \citep{yuAvailabilityAttacksCreate2022}.} 
This type of noise aims to introduce some linearly separable high-frequent patterns by adding class-wise hand-crafted repeated noisy patches. Specifically, this method synthesizes some simple data points for an N-class classification problem, where three stages are included: class-wise initial points generation, duplicating for local correlation, and the final magnitude rescaling. For the noise frame size, we set it as 8 by default; for the norm size, we set it as 8/255 by default. For other settings, please refer to the original paper accordingly for other details. 

\begin{algorithm}[thbp]
\caption{Calculation of surrogate model's robustness $\mathcal{R} _{\theta}$}

\begin{algorithmic}[1]
\REQUIRE Trained defensive noise $\delta^u$, Training data set $\mathcal{T}$, Training steps $M$, attack PGD parameters $\rho_a$, $\alpha_a$ and $K_a$, transformation distribution $T$
\ENSURE The robustness of surrogate model $\mathcal{R} _{\theta}$
    \STATE $\mathcal{R} _{\theta} \gets 0$
    \FOR{$i$ \textbf{in} $1, \cdots, M$}
        \STATE Sample a minibatch $(x, y) \sim \mathcal{T}$,
        \STATE Sample a transformation function $t \sim T$.
        \STATE $\delta^a \leftarrow \mathrm{PGD}(t(x+\delta^u),y,f'_\theta,\rho_a,\alpha_a,K_a)$ 
		\STATE Update source model $f'_\theta$ based on minibatch $(t(x+\delta^u)+\delta^a,y)$
        \STATE $\mathcal{R} _{\theta} \gets \mathcal{R} _{\theta} + \ell(f^\prime_\theta(t(x+\delta^u)+\delta^a),y)$
    \ENDFOR
\end{algorithmic}
\label{alg: robustness model}
\end{algorithm}

\begin{algorithm}[htbp]
\caption{Calculation of defensive noise's robustness $\mathcal{R} _{\delta^u}$}
\begin{algorithmic}[1]
\REQUIRE
Trained surrogate model $\theta$, Training data set $\mathcal{T}$, Training steps $M$, defense PGD parameters $\rho_u$, $\alpha_u$ and $K_u$, attack PGD parameters $\rho_a$, $\alpha_a$ and $K_a$, transformation distribution $T$, the sampling number $J$ for gradient approximation
\ENSURE The robustness of defensive noise $\mathcal{R} _{\delta^u}$
\STATE $\mathcal{R} _{\delta^u} \gets 0$
    \FOR{$i$ \textbf{in} $1, \cdots, M$}
        \STATE Sample a minibatch $(x, y) \sim \mathcal{T}$
        \STATE Randomly initialize $\delta^u$.
        \FOR{$k$ \textbf{in} $1,\cdots, K_u$}
            \FOR{$j$ \textbf{in} $1,\cdots, J$}
        \STATE $\delta^a_j \leftarrow \mathrm{PGD}(t_j(x+\delta^u),y,f'_\theta,\rho_a,\alpha_a,K_a)$ 
                \STATE Sample a transformation function $t_j \sim T$.
            \ENDFOR
            \STATE $g_k \leftarrow \frac{1}{J} \sum_{j=1}^J \frac{\partial}{\partial \delta^u} \ell(f'_\theta(t_j(x+\delta^u) + \delta^a_j), y)$
            \STATE $\delta^u \leftarrow \prod_{\|\delta\|\leq\rho_u} \left( \delta^u - \alpha_u \cdot \mathrm{sign}(g_k) \right)$
        \ENDFOR
        \STATE Sample a transformation function $t \sim T$.
        \STATE $\delta^a \leftarrow \mathrm{PGD}(t(x+\delta^u),y,f'_\theta,\rho_a,\alpha_a,K_a)$ 
            \STATE $\mathcal{R} _{\delta^u} \gets \mathcal{R} _{\delta^u} + \ell(f^\prime_\theta(t(x+\delta^u)+\delta^a),y)$
    \ENDFOR
\end{algorithmic}
\label{alg: robustness noise}
\end{algorithm}

\section{Implementation Details}
\label{app. imple}

\subsection{Hardware and Software Details}
The experiments were conducted on an Ubuntu 20.04.6 LTS (focal) environment with 503GB RAM. And the experiments on CIFAR-10 and CIFAR-100 are conducted on 1 GPU (NVIDIA\textsuperscript{\textregistered} RTX\textsuperscript{\textregistered} A5000 24GB) and 32 CPU cores (Intel\textsuperscript{\textregistered} Xeon\textsuperscript{\textregistered} Silver 4314 CPU @ 2.40GHz). For the ImageNet, experiments are conducted on 2 GPU (NVIDIA\textsuperscript{\textregistered} RTX\textsuperscript{\textregistered} A5000 24GB) and with the same CPU configuration. Python 3.10.12 and Pytorch 2.0.1 are used for all the implementations.

\subsection{Method Details}
\header{Noise Regeneration.} \label{app. alg}
Following existing practices, after training a noise generator following Alg. \ref{alg: sem-train}, we leverage the generator to conduct a final round for generating defensive noise via Alg. \ref{alg: sem-gen}. We found that this trick can somewhat boost protection performance compared to leveraging defensive noise generated in the previous round or a weighted sum version. 
\begin{algorithm}[htbp]
\caption{Densive Noise Generation of SEM approach}
\label{alg: sem-gen}
\begin{algorithmic}[2]
\REQUIRE
Noise generator $f'_\theta$, Training data set $\mathcal{T}$, defense PGD parameters $\rho_u$, $\alpha_u$, attack PGD parameters $\rho_a$, $\alpha_a$, transformation distribution $T$, the sampling number $J$ for gradient approximation
\ENSURE Robust unlearnable example $\mathcal{T}^{\prime}=\{(x_i^{\prime},y_i)\}_{i=1}^{n}$
    \FOR{$(x_i, y_i)$ \textbf{in} $\mathcal{T}$}
    \STATE Randomly initialize $\delta^u_i$.
    \FOR{$k$ \textbf{in} $1,\cdots, K_u$}
            \FOR{$j$ \textbf{in} $1,\cdots, J$}
                \STATE {Random sample noise $\delta^r_j \sim \mathcal{P}$.}
                \STATE Sample a transformation function $t_j \sim T$.
            \ENDFOR
\STATE $g_{k}\leftarrow \frac{1}{J}\sum _{j=1}^{J}\frac{\partial }{\partial \delta ^{u}} \ell (f'_{\theta } (t_{j} (x+\delta ^{u} )+\delta _{j}^{r} ),y)$
            \STATE $\delta^u_i \leftarrow \prod_{\|\delta^u_i\|\leq\rho_u} \left( \delta^u_i - \alpha_u \cdot \mathrm{sign}(g_k) \right)$
        \ENDFOR
    \STATE $x_i^{\prime} \gets x_i + \delta^u_i$
    \ENDFOR
\end{algorithmic}
\label{gen}
\end{algorithm}


\header{Data Augmentation.}~Following~\cite{fuRobustUnlearnableExamples2022}, we use different data augmentations for different datasets during training.
For CIFAR-10 and CIFAR-100, we perform the following data augmentations: random flipping, the padding on each size with $4$ pixels, random cropping the image to $32 \times 32$, and then rescaling per pixel to $[-0.5, 0.5]$ for each image.
For the ImageNet subset, data augmentation is conducted via RandomResizedCrop operation with $224 \times 224 $, random flipping, and then rescaling image with range $[-0.5, 0.5]$.
\\
\textbf{Adversarial Training.}~We focus on $L_\infty$-bounded noise $\norm{\rho_a}_\infty \leq \rho_a$ in adversarial training. Note that defensive noise radius $\rho_u$ is set to $8/255$, and adversarial training radius $\rho_a$ is set to $4/255$ for each dataset by default. And for the adversarial training radius setting rule, we follow the setting of REM~\cite{fuRobustUnlearnableExamples2022} with $\rho_a \leq \rho_u$. 
The detailed parameter setting is presented in \ref{tab:adv}, where $\alpha_d$, $\alpha_a$ denote the step size of PGD for defensive and adversarial noise generation $K_{\cdot}$ denotes the corresponding PGD step number.

\begin{table}[htbp]
	\centering
	\caption{
	The settings of PGD for the noise generations of different methods. $\rho_u$ denotes the defensive perturbation radius of different types of noise, while $\rho_a$ denotes the adversarial perturbation radius of the REM and SEM.
	}
	\label{tab: pgd_params}
\resizebox{\linewidth}{!}{
	\begin{tabular}{cccccc} 
\toprule
Datasets                                                                       & Noise Type     & $\alpha_d$   & $K_d$ & $\alpha_a$ & $K_a$  \\ 
\midrule
\multirow{3}{*}{\begin{tabular}[c]{@{}c@{}}CIFAR-10 \\ CIFAR-100\\Facescrub\end{tabular}} & EM             & $\rho_u/5$   & $10$  & -          & -      \\
                                                                               & TAP            & $\rho_u/125$ & $250$ & -          & -      \\
                                            
                                                                               & REM \& SEM(Ours) & $\rho_u/5$   & $10$  & $\rho_a/5$ & $10$   \\ 
\midrule
\multirow{3}{*}{ImageNet Subset}                                               & EM             & $\rho_u/5$   & $7$   & -          & -      \\
                                                                               & TAP            & $\rho_u/50$  & $100$ & -          & -      \\                                                                               & REM \& SEM(Ours) & $\rho_u/4$   & $7$   & $\rho_a/5$ & $10$   \\
\bottomrule
\end{tabular}
}

\label{tab:adv}
\end{table}



\header{Metrics.}~We use the test accuracy in the evaluation phase to measure the data protection ability of the defensive noise. A low test accuracy suggests that the model is fooled by defensive noise to have poor generalization ability. Thus it learned little information from the protected data, which implies better data protection effectiveness of the noise. We use the generation time of the defensive noise to evaluate efficiency. To exclude random error, all results are repeated three times with different random seeds $s \in \{0, 6178, 42\}$.


\begin{table}[htbp]
\centering
\caption{Correlation of the protection performance and two robustness terms under three different metrics.}
\label{tab: corr more}
\begin{tabular}{cccc} 
\toprule
                         & Pearson & Spearman & Kendall  \\
                         \midrule
$\mathcal{R}_{\theta}$   &  0.9551      & 	0.8422         &   0.6921       \\
$\mathcal{R}_{\delta_u}$ &   -0.5851      &    -0.7003      &   -0.5407       \\
\bottomrule
\end{tabular}
\end{table}

\section{More Results}
\label{app. vis}
We visualize more defensive noise and the corresponding unlearnable images in Fig. \ref{fig:more-vis-example}, Fig. \ref{fig:more-vis-example-cifar100}, Fig. \ref{fig:more-vis-img-sub} and Fig. \ref{fig:more-vis-face}. For the main table that compares different methods under various adversarial training radii, we present an extended version of Tab \ref{tab:diff_rad} with standard deviation in Tab. \ref{tab:diff_rad_def8_more} and Tab. \ref{tab:diff_rad_def16_more}. 

\header{Effectiveness under partial poisoning.}   
Under a more challenging realistic learning setting where only a portion of data can be protected by the defensive noise, we also study the effect of different poisoning ratios on the test accuracy of models trained on the poisoned dataset. From Tab. \ref{tab: diff_perc}, we can see that our method outperforms other baselines when the poisoning rate is equal to or higher than $80\%$, and 
the test accuracy of models trained on the poisoned dataset decreases as the poisoning ratio increases. However, under lower poisoning rate settings, where the model might learn from clean examples, all the existing noise types fail to degrade the model performance significantly. This result indicates a common limitation of existing unlearnable methods, calling for further studies investigating this challenging setting. 

\begin{table*}[t]
    \renewcommand{\arraystretch}{0.7}
    \centering
    \caption{Test accuracy (\%) on CIFAR-10 with different protection percentages. We set the defensive perturbation radius $\rho_u$ as $8/255$. For REM and SEM, we set $\rho_a$ and $\rho_a$ as $4/255$. 
    }
    \label{tab:diff_perc_c10}
    \resizebox{.8\linewidth}{!}{
\begin{tabular}{ccc|cc|cc|cc|cc|c} 
\toprule
\multirow{3}{*}{\begin{tabular}[c]{@{}c@{}}Adv. \\ Train. \\ $\rho_a$\end{tabular}} & \multirow{3}{*}{\begin{tabular}[c]{@{}c@{}}Noise \\ Type\end{tabular}} & \multicolumn{10}{c}{Data Protection Percentage}                                                                                                                                                                                   \\ 
\cmidrule{3-12}
                                                                                    &                                                                        & \multirow{2}{*}{0\%}   & \multicolumn{2}{c|}{20\%}               & \multicolumn{2}{c|}{40\%}               & \multicolumn{2}{c|}{60\%}               & \multicolumn{2}{c|}{80\%}                        & \multirow{2}{*}{100\%}  \\
                                                                                    &                                                                        &                        & Mixed          & Clean                  & Mixed          & Clean                  & Mixed          & Clean                  & Mixed                   & Clean                  &                         \\ 
\midrule
\multirow{6}{*}{$2/255$}                                                            & EM                                                                     & \multirow{6}{*}{92.37} & 92.26          & \multirow{6}{*}{91.30} & 91.94          & \multirow{6}{*}{90.31} & 91.81          & \multirow{6}{*}{88.65} & 91.14                   & \multirow{6}{*}{83.37} & 71.43                   \\
                                                                                    & TAP                                                                    &                        & \textbf{92.17} &                        & 91.62          &                        & 91.32          &                        & 91.48                   &                        & 90.53                   \\
                                                                                    & NTGA                                                                   &                        & 92.41          &                        & 92.19          &                        & 92.23          &                        & 91.74                   &                        & 85.13                   \\
                                                                                    & SC                                                                     &                        & 92.3           &                        & 91.98          &                        & 90.17          &                        & 87.58                   &                        & 47.66                   \\
                                                                                    & REM                                                                    &                        & 92.36          &                        & \textbf{90.22} &                        & \textbf{88.45} &                        & 84.82                   &                        & 30.69                   \\
                                                                                    & SEM                                                                    &                        & 91.94          &                        & 90.89          &                        & 88.90          &                        & \textbf{\textbf{82.98}} &                        & \textbf{11.32}          \\ 
\midrule
\multirow{6}{*}{$4/255$}                                                            & EM                                                                     & \multirow{6}{*}{89.51} & 89.60          & \multirow{6}{*}{88.17} & 89.40          & \multirow{6}{*}{86.76} & 89.49          & \multirow{6}{*}{85.07} & 89.10                   & \multirow{6}{*}{79.41} & 88.62                   \\
                                                                                    & TAP                                                                    &                        & \textbf{89.01} &                        & \textbf{88.66} &                        & \textbf{88.40} &                        & 88.04                   &                        & 88.02                   \\
                                                                                    & NTGA                                                                   &                        & 89.56          &                        & 89.35          &                        & 89.22          &                        & 89.17                   &                        & 88.96                   \\
                                                                                    & SC                                                                     &                        & 89.4           &                        & 89.14          &                        & 89.23          &                        & 89.02                   &                        & 84.43                   \\
                                                                                    & REM                                                                    &                        & 89.60          &                        & 89.34          &                        & 89.61          &                        & 88.09                   &                        & 48.16                   \\
                                                                                    & SEM                                                                    &                        & 89.82          &                        & 89.65          &                        & 88.66          &                        & \textbf{85.19}          &                        & \textbf{31.29}          \\
\midrule
\end{tabular}

    }
    
    \label{tab: diff_perc}
    \end{table*}

\header{Robustness terms $\mathcal{R} _{\theta}$ and $ \mathcal{R} _{\delta ^u}$. } To solve Eq. \ref{eq: surrogate robustness} and Eq. \ref{eq: noise robustness}, we conduct iterative training following Alg. \ref{alg: robustness model} and Alg. \ref{alg: robustness noise} with trained model or defensive noise. Moreover, we present three metrics (Pearson, Spearman, and Kendall) in Tab. \ref{tab: corr more} to quantify the correlation between two robustness terms and the protection performance. These metrics, each unique in its way, provide diverse perspectives on the correlations, be it linear relationships or rank-based associations. The results show that \( \mathcal{R}_{\theta} \) exhibits a strong positive correlation with the protection performance across all three metrics. This suggests that as \( \mathcal{R}_{\theta} \) increases, the protection performance also increases. \( \mathcal{R}_{\delta_u} \), on the other hand, shows a moderate negative correlation with the protection performance across all three metrics. 

\begin{table*}[htbp]
    \centering
    \caption{Test accuracy (\%) of models trained on data protected by different defensive noises via adversarial training with different perturbation radii.
    The defensive perturbation radius $\rho_u$ is set as $8/255$ for every type of noise, while the adversarial perturbation radius $\rho_a$ of REM noise and SEM noise take various values.}
    \label{tab:diff_rad_def8_more}
\resizebox{\linewidth}{!}{
\begin{tabular}{cccccccccccccccc} 
\toprule
\multirow{2}{*}{Dataset}   & \multirow{2}{*}{\begin{tabular}[c]{@{}c@{}}Adv. Train. \\ $\rho_a$\end{tabular}} & \multirow{2}{*}{Clean} & \multirow{2}{*}{EM} & \multirow{2}{*}{TAP} & \multirow{2}{*}{NTGA} & \multicolumn{5}{c}{REM}                              & \multicolumn{5}{c}{SEM}                                                                                                           \\ 
\cmidrule(l){7-16}
                           &                                                                                  &                        &                     &                      &                       & $\rho_a = 0$ & $1/255$ & $2/255$ & $3/255$ & $4/255$ & $\rho_a = 0$     & $1/255$                   & $2/255$                   & $3/255$                   & $4/255$                    \\ 
\midrule
\multirow{5}{*}{CIFAR-10}  & $0$                                                                              & 94.66                  & 13.20               & 22.51                & 16.27                 & 15.18        & 13.05   & 20.60   & 20.67   & 27.09   & 15.87 $\pm$ 0.50 & 11.41 $\pm$ 0.31          & \textbf{10.42 $\pm$ 0.36} & 16.83 $\pm$ 1.76          & 10.52 $\pm$ 1.41           \\
                           & $1/255$                                                                          & 93.74                  & 22.08               & 92.16                & 41.53                 & 27.20        & 14.28   & 22.60   & 25.11   & 28.21   & 16.59 $\pm$ 0.14 & \textbf{12.11 $\pm$ 0.07} & 15.80 $\pm$ 0.21          & 17.92 $\pm$ 0.49          & 12.16 $\pm$ 1.30           \\
                           & $2/255$                                                                          & 92.37                  & 71.43               & 90.53                & 85.13                 & 75.42        & 29.78   & 25.41   & 27.29   & 30.69   & 64.24 $\pm$ 1.45 & 13.44 $\pm$ 0.21          & \textbf{11.49 $\pm$ 0.10} & 19.13 $\pm$ 1.11          & 21.68 $\pm$ 1.32           \\
                           & $3/255$                                                                          & 90.90                  & 87.71               & 89.55                & 89.41                 & 88.08        & 73.08   & 46.18   & 30.85   & 35.80   & 89.18 $\pm$ 0.19 & 28.6 $\pm$ 1.06           & 21.88 $\pm$ 0.62          & \textbf{20.91 $\pm$ 0.54} & 23.43 $\pm$ 1.44           \\
                           & $4/255$                                                                          & 89.51                  & 88.62               & 88.02                & 88.96                 & 89.15        & 86.34   & 75.14   & 47.51   & 48.16   & 89.30 $\pm$ 0.08 & 67.52 $\pm$ 1.62          & 35.63 $\pm$ 3.30          & 62.09 $\pm$ 0.93          & \textbf{31.92 $\pm$ 1.53}  \\ 
\midrule
\multirow{5}{*}{CIFAR-100} & $0$                                                                              & 76.27                  & \textbf{1.60}       & 13.75                & 3.22                  & 1.89         & 3.72    & 3.03    & 8.31    & 10.14   & 1.69 $\pm$ 0.03  & 1.95 $\pm$ 0.09           & 3.71 $\pm$ 0.19           & 3.57 $\pm$ 0.29           & 7.27 $\pm$ 0.44            \\
                           & $1/255$                                                                          & 71.90                  & 71.47               & 70.03                & 65.74                 & 9.45         & 4.47    & 5.68    & 9.86    & 11.99   & 9.68 $\pm$ 0.52  & \textbf{2.51 $\pm$ 0.13}  & 4.17 $\pm$ 0.03           & 5.31 {\small $\pm$ 0.32}                       & 10.07 $\pm$ 0.21           \\
                           & $2/255$                                                                          & 68.91                  & 68.49               & 66.91                & 66.53                 & 52.46        & 13.36   & 7.03    & 11.32   & 14.15   & 48.25 $\pm$ 0.51 & 10.45 $\pm$ 0.37          & \textbf{4.92 $\pm$ 0.16}  & 6.06 $\pm$ 0.19           & 10.95 $\pm$ 0.36           \\
                           & $3/255$                                                                          & 66.45                  & 65.66               & 64.30                & 64.80                 & 66.27        & 44.93   & 27.29   & 17.55   & 17.74   & 63.53 $\pm$ 0.21 & 42.20 $\pm$ 0.35          & 23.85 $\pm$ 0.14          & \textbf{7.74 $\pm$ 0.17}  & 12.04 $\pm$ 0.35           \\
                           & $4/255$                                                                          & 64.50                  & 63.43               & 62.39                & 62.44                 & 64.17        & 61.70   & 61.88   & 41.43   & 27.10   & 64.35 $\pm$ 0.18 & 58.57 $\pm$ 0.43          & 53.68 $\pm$ 0.92          & 28.61 $\pm$ 0.43          & \textbf{19.77 $\pm$ 0.62}  \\
\bottomrule
\end{tabular}
}
    \vspace{-4mm}
    \end{table*}

\begin{table*}[htbp]
    \centering
    \caption{Test accuracy (\%) of models trained on data that are protected by different defensive noises via adversarial training with different perturbation radii.
    The defensive perturbation radius $\rho_u$ is set as $16/255$ for every type of noise, while the adversarial perturbation radius $\rho_a$ of REM noise and SEM noise take various values.}
    \label{tab:diff_rad_def16_more}
\resizebox{\linewidth}{!}{   
\begin{tabular}{cccccccccccccccc} 
\toprule
\multirow{2}{*}{Dataset}   & \multirow{2}{*}{\begin{tabular}[c]{@{}c@{}}Adv. Train. \\ $\rho_a$\end{tabular}} & \multirow{2}{*}{Clean} & \multirow{2}{*}{EM} & \multirow{2}{*}{TAP} & \multirow{2}{*}{NTGA} & \multicolumn{5}{c}{REM}                              & \multicolumn{5}{c}{SEM}                                                                                                 \\ 
\cmidrule(l){7-16}
                           &                                                                                  &                        &                     &                      &                       & $\rho_a = 0$ & $2/255$ & $4/255$ & $6/255$ & $8/255$ & $\rho_a = 0$             & $2/255$          & $4/255$                   & $6/255$          & $8/255$                    \\ 
\midrule
\multirow{5}{*}{CIFAR-10}  & $0$                                                                              & 94.66                  & 16.84               & 11.29                & \textbf{10.91}        & 13.69        & 13.15   & 19.51   & 24.54   & 24.08   & 15.12 $\pm$ 0.32         & 14.49 $\pm$ 0.13 & 12.64 $\pm$ 1.66          & 13.52 $\pm$ 2.58 & 17.96 $\pm$ 1.99           \\
                           & $2/255$                                                                          & 92.37                  & 22.46               & 78.01                & 19.96                 & 21.17        & 17.73   & 20.46   & 21.89   & 26.30   & 24.58 $\pm$ 0.43         & 15.80 $\pm$ 0.14 & \textbf{13.82 $\pm$ 0.61} & 17.81 $\pm$ 1.48 & 19.59 $\pm$ 1.63           \\
                           & $4/255$                                                                          & 89.51                  & 41.95               & 87.60                & 32.80                 & 45.87        & 31.18   & 23.52   & 26.11   & 28.31   & 63.91 {\small $\pm$ 1.04} & 18.25 {\small $\pm$ 0.32} & \textbf{15.45 {\small $\pm$ 0.35}} & 17.36 {\small $\pm$ 0.63} & 21.15 {\small $\pm$ 4.53} \\
                           & $6/255$                                                                          & 86.90                  & 52.13               & 85.44                & 60.64                 & 66.82        & 58.37   & 40.25   & 30.93   & 31.50   & 80.07 {\small $\pm$ 0.28} & 31.88 {\small $\pm$ 0.91} & \textbf{18.14 {\small $\pm$ 0.43}} & 22.59 {\small $\pm$ 1.18} & 22.47 {\small $\pm$ 2.43} \\
                           & $8/255$                                                                          & 84.79                  & 64.71               & 82.56                & 74.50                 & 77.08        & 72.87   & 63.49   & 46.92   & 36.37   & 83.12 $\pm$ 0.43         & 58.10 $\pm$ 0.99 & 40.58 $\pm$ 1.03          & 33.24 {\small $\pm$ 1.38}             & \textbf{31.57 $\pm$ 1.61}  \\ 
\midrule
\multirow{5}{*}{CIFAR-100} & $0$                                                                              & 76.27                  & 1.44                & 4.84                 & 1.54                  & 2.14         & 4.16    & 4.27    & 5.86    & 11.16   & \textbf{1.22 $\pm$ 0.04} & 1.30 $\pm$ 0.06  & 1.94 $\pm$ 0.29           & 3.45 $\pm$ 0.94  & 5.12 $\pm$ 0.42            \\
                           & $2/255$                                                                          & 68.91                  & 5.21                & 64.59                & 5.21                  & 6.68         & 5.04    & 4.83    & 7.86    & 13.42   & 3.71 $\pm$ 0.16          & 1.70 $\pm$ 0.08  & \textbf{1.59 $\pm$ 0.05}  & 6.69 $\pm$ 0.38  & 9.17 $\pm$ 0.22            \\
                           & $4/255$                                                                          & 64.50                  & 35.65               & 61.48                & 18.43                 & 28.27        & 9.80    & 6.87    & 9.06    & 14.46   & 26.87 $\pm$ 0.39         & 4.83 $\pm$ 0.16  & \textbf{1.81 $\pm$ 0.21}  & 7.48 $\pm$ 0.08  & 10.29 $\pm$ 0.25           \\
                           & $6/255$                                                                          & 60.86                  & 56.73               & 57.66                & 46.30                 & 47.08        & 35.25   & 30.16   & 14.41   & 17.67   & 46.88 $\pm$ 0.27         & 25.14 $\pm$ 0.44 & \textbf{6.91 $\pm$ 0.22}  & 8.12 $\pm$ 0.36  & 12.22 $\pm$ 0.28           \\
                           & $8/255$                                                                          & 58.27                  & 56.66               & 55.30                & 50.81                 & 54.23        & 49.82   & 54.55   & 33.86   & 23.29   & 53.01 $\pm$ 0.54         & 41.34 $\pm$ 0.43 & 26.04 $\pm$ 0.34          & 25.46 $\pm$ 0.23 & \textbf{14.24 $\pm$ 0.16}  \\
\bottomrule
\end{tabular}

    }
\end{table*}

\header{Effect of $\ell_\infty$ attacking algorithms} {We provide additional results on REM with different $\ell_\infty$-norm adversarial attack algorithms, including FGSM, TPGD, and APGD} (from \texttt{torchattacks}). As shown in Tab. \ref{tab:diff-attack}, results suggest that they are still inferior to SEM, indicating the limitation root in the REM optimization.

\begin{table}[H]
    \small
        \centering
        
        \caption{
            REM with different $\ell_\infty$ attacking algorithms. 
        }
        \label{tab:diff-attack}
        \resizebox{\linewidth}{!}{
            \begin{tabular}{cccccc} 
                \toprule
Methods &SEM & REM+PGD & REM+FGSM & REM+TPGD                                  & REM+APGD \\
                \midrule
Test Acc. ($\downarrow$) &               \textbf{ 29.53 }&48.16   & 40.95 & 33.19 & 80.87                        \\
                \bottomrule
                \end{tabular}
        }
\end{table}

\header{Generation time. } We present the defensive noise generation time cost using REM and SEM. The results in Tab. \ref{tab:def_noise_time_cost} suggest that SEM is $\sim $ 4 times faster than the REM approach.

\begin{table}[H]
\centering
\caption{Time costs comparison for generating defensive noise between REM and SEM.}
\label{tab:def_noise_time_cost}
\resizebox{1.0\linewidth}{!}{
\begin{tabular}{ccccc} 
\toprule
Dataset          & REM   & SEM(Ours) & Reduce (↓) & \multicolumn{1}{l}{Speedup ($\times$)}  \\ 
\midrule
CIFAR-10/100         & 22.6h & \textbf{5.78h}     & -16.82h    & 3.91                                    \\
ImageNet  Subset & 5.2h  & \textbf{1.33h}     & -3.87h     & 3.90                                    \\
\bottomrule
\end{tabular}

}
\end{table}

\section{Understanding SEM Noise via UMAP 2D Projection}
\label{app. under}
In this section, we uncover more insights into the effectiveness of the stable error-minimizing noise with the UMAP 2D Projection method. Our term of ``stability'' of noise mainly refers to their resistance against \textit{more diverse perturbation} during the evaluation phase. Based on our empirical observation, the adversarial perturbations in the REM optimization turn out to be monotonous and not aligned with the more diverse perturbations during evaluation. This phenomenon might be caused by the mismatched training dynamic since, during the noise training phase, the defensive noise can be updated while it is fixed during evaluation. {We term our method ``stable'' error-minimizing since we leverage more diverse random perturbation to train defensive noise based on our insights that the actual perturbation during the evaluation phase is more aligned with random perturbation.} Based on Eq. 9, we further define stability as follows: 
\begin{definition}[Delusiveness of noise]
    \label{def:delusiveness}
    Given defensive noise $\mathbf{\delta^u}$
     we define its delusiveness $\mathcal{R}_D (\mathbf{\delta^u})$ as the testing loss of the model $f_{\theta^*}$ that trained with data transformations:
    \begin{align}
        \label{eq:delusiveness}
        \small D(\mathbf{\delta^u})=\underset{(x,y)\sim \mathcal{D}}{\mathbb{E}}\left[ \mathcal{L} \left( f_{\theta ^*}(x),y \right) \right] ,\\
        \small 
        \mathrm{s}.\mathrm{t}. \theta ^*=\underset{\theta}{\mathrm{arg}\min}\sum_{\left( x_i,y_i \right) \in \mathcal{T}}{\underset{t\sim T}{\mathbb{E}}\mathcal{L} \left( f_{\theta}\left( t(x_i+\delta _{i}^{u}) \right) ,y_i \right)}.
    \end{align}
\end{definition}
\vspace{-7pt}
\begin{definition}[Stability of noise]
    \label{def:stability}
     Given $\mathbf{\delta^u}$ and adversarial perturbation $\mathbf{\delta^a}$ with radius $\rho_a$, its stability $\mathcal{S}(\mathbf{\delta^u}, \rho_a)$ is opposite to its worst-case delusiveness under perturbation:
    \begin{align}
        \label{eq:stability}
        \small 
        \mathcal{S}(\mathbf{\delta^u}, \rho_a)=\underset{||\delta _{i}^{a}||_p\le \rho _a}{\min} \mathcal{R}_D(\mathbf{\delta^u}+\mathbf{\delta^a}).
    \end{align}
\end{definition}
\vspace{-7pt}
\begin{figure}[thbp]
    \centering
    \includegraphics[width=\linewidth]{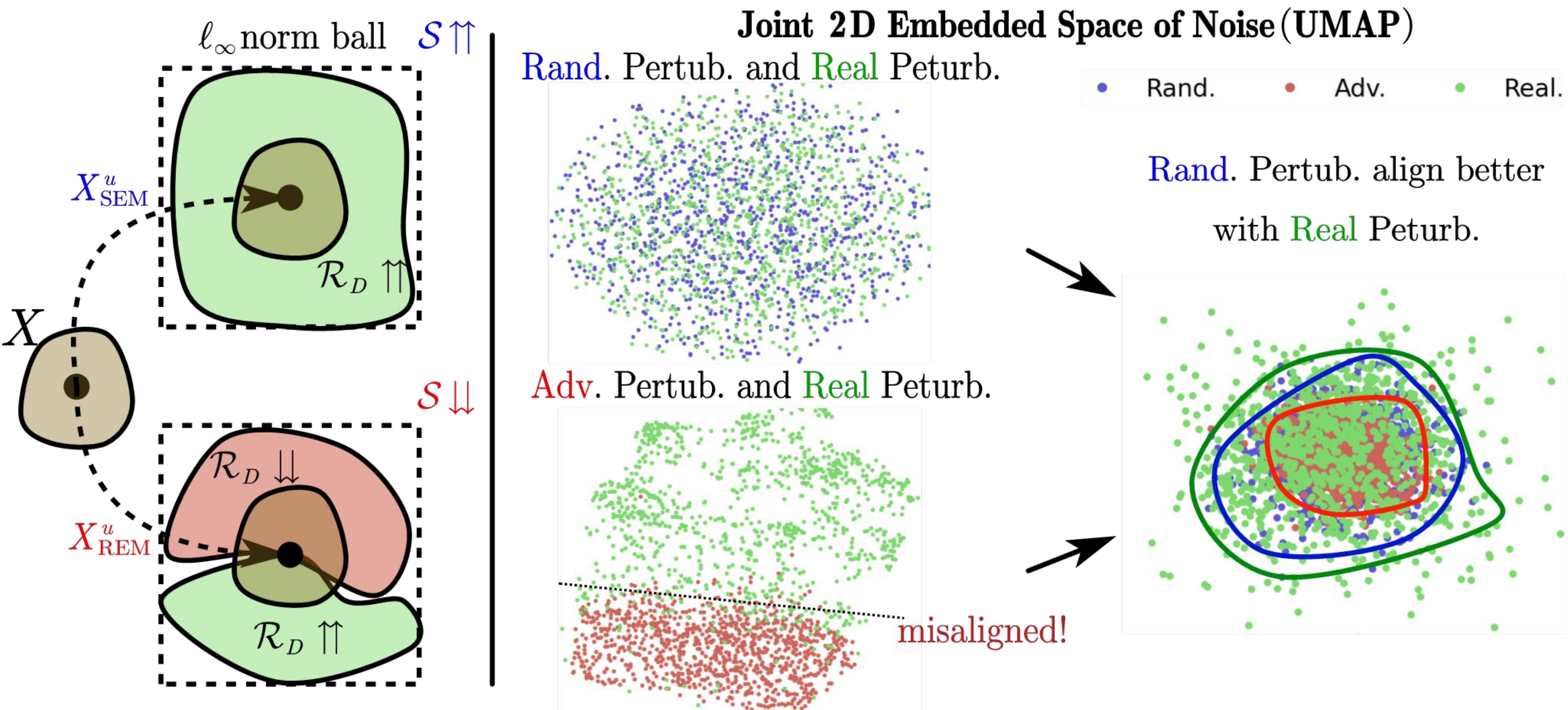}
    \caption{
        Understanding of the superiority of SEM compared to REM with UMAP visualization of noise.  
    }
    \label{fig:understanding}
\end{figure}
To validate our statements, we conduct UMAP visualization of the random perturbation and adversarial perturbation in the REM training and actual evaluation phase. As shown in Fig. \ref{fig:understanding}, adversarial perturbations in REM turn out to be monotonous, and thus, the induced defensive noise is not ``stable'' when the actual perturbation falls outside the specific region. In contrast, the random perturbation in SEM is more aligned with the real perturbation during evaluation.

\begin{figure}[htbp]
\addfigs{cifar10}{rad-8}{1250}{}{1.0}
\hspace{0.2cm}
\addfigs{cifar10}{rad-8}{1251}{}{1.0}
\vspace{0.4cm}
\addfigs{cifar10}{rad-8}{1252}{}{1.0}
\hspace{0.2cm}
\addfigs{cifar10}{rad-8}{1253}{}{1.0}
\addfigs{cifar10}{rad-8}{1257}{}{1.0}
\addfigs{cifar10}{rad-8}{1258}{}{1.0}
\caption{
More Visualization results on CIFAR-10. Examples of data are protected by error-minimizing noise (EM), targeted adversarial poisoning noise (TAP), shortcut noise (SC), robust error-minimizing noise (REM), and stable error-minimizing noise (SEM).
Note that $\rho_u$ is set as $8/255$, and the adversarial perturbation radius $\rho_a$ is set as $4/255$.
}
\label{fig:more-vis-example}
\end{figure}
\begin{figure}[htbp]
\addfigs{cifar100}{rad-8}{1264}{}{1.0}
\hspace{0.2cm}
\addfigs{cifar100}{rad-8}{1265}{}{1.0}
\vspace{0.4cm}
\addfigs{cifar100}{rad-8}{1266}{}{1.0}
\hspace{0.2cm}
\addfigs{cifar100}{rad-8}{1253}{}{1.0}
\vspace{0.4cm}
\addfigs{cifar100}{rad-8}{1254}{}{1.0}
\hspace{0.2cm}
\addfigs{cifar100}{rad-8}{1255}{}{1.0}
\caption{
More Visualization results on CIFAR-100. Examples of data are protected by error-minimizing noise (EM), targeted adversarial poisoning noise (TAP), shortcut noise (SC), robust error-minimizing noise (REM), and stable error-minimizing noise (SEM).
Note that $\rho_u$ is set as $8/255$ and the adversarial perturbation radius $\rho_a$ is set as $4/255$.
}
\label{fig:more-vis-example-cifar100}
\end{figure}
\begin{figure}[htbp]
\addfigs{img-three}{rad-8}{1000}{}{1.0}
\hspace{0.2cm}
\addfigs{img-three}{rad-8}{1001}{}{1.0}
\vspace{0.4cm}
\addfigs{img-three}{rad-8}{1052}{}{1.0}
\hspace{0.2cm}
\addfigs{img-three}{rad-8}{1152}{}{1.0}
\vspace{0.4cm}
\addfigs{img-three}{rad-8}{1223}{}{1.0}
\hspace{0.2cm}
\addfigs{img-three}{rad-8}{1259}{}{1.0}
\vspace{0.3cm}
\caption{
More Visualization results on ImageNet Subset. Examples of data are protected by error-minimizing noise (EM), targeted adversarial poisoning noise (TAP), shortcut noise (SC), robust error-minimizing noise (REM), and stable error-minimizing noise (SEM).
Note that $\rho_u$ is set as $8/255$, and the adversarial perturbation radius $\rho_a$ is set as $4/255$.
}
\label{fig:more-vis-img-sub}
\end{figure}
\begin{figure}[htbp]
\addfigs{face-10}{rad-8}{6}{}{1.0}
\hspace{0.2cm}
\addfigs{face-10}{rad-8}{42}{}{1.0}
\vspace{0.4cm}
\addfigs{face-10}{rad-8}{58}{}{1.0}
\hspace{0.2cm}
\addfigs{face-10}{rad-8}{56}{}{1.0}
\vspace{0.4cm}
\addfigs{face-10}{rad-8}{73}{}{1.0}
\hspace{0.2cm}
\addfigs{face-10}{rad-8}{66}{}{1.0}
\vspace{0.3cm}
\caption{
More Visualization results on the Facescrub dataset. Examples of data are protected by error-minimizing noise (EM), targeted adversarial poisoning noise (TAP), shortcut noise (SC), robust error-minimizing noise (REM), and stable error-minimizing noise (SEM).
Note that $\rho_u$ is set as $8/255$, and the adversarial perturbation radius $\rho_a$ is set as $4/255$.
}
\label{fig:more-vis-face}
\end{figure}

\fi

\end{document}